\documentclass{article}

\usepackage[final,nonatbib]{neurips_glfrontiers_2022}

\usepackage[utf8]{inputenc} 
\usepackage[T1]{fontenc}    
\usepackage{hyperref}       
\usepackage{url}            
\usepackage{booktabs}       
\usepackage{amsfonts}       
\usepackage{nicefrac}       
\usepackage{microtype}      
\usepackage{xcolor}         

\usepackage{amssymb,bm,amsthm}
\usepackage{enumerate}
\usepackage{color, colortbl}
\usepackage{booktabs, multirow}
\definecolor{darkgreen}{rgb}{0, 0.5, 0}
\definecolor{red}{rgb}{1, 0, 0}
\definecolor{purple}{rgb}{0.5, 0, 0.5}
\usepackage{chngpage}
\usepackage{comment}
\usepackage{mfirstuc}
\usepackage{mathtools}
\usepackage{lettrine}

\newcommand\ie{\textit{i.e.,}}
\newcommand\eg{\textit{e.g.,}}

\newcommand\etc{\textit{etc.}}

\newcommand{\norm}[1]{\left\lVert#1\right\rVert}

\newcommand{\beq}{\begin{equation}}
\newcommand{\eeq}{\end{equation}}
\newcommand{\beqnn}{\begin{equation*}}
\newcommand{\eeqnn}{\end{equation*}}
\newcommand{\beqy}{\begin{eqnarray}}
\newcommand{\eeqy}{\end{eqnarray}}
\newcommand{\beqynn}{\begin{eqnarray*}}
\newcommand{\eeqynn}{\end{eqnarray*}}
\newcommand{\bit}{\begin{itemize}}
\newcommand{\eit}{\end{itemize}}
\newcommand{\ben}{\begin{enumerate}}
\newcommand{\een}{\end{enumerate}}
\newcommand{\bex}{\begin{example}}
\newcommand{\eex}{\end{example}}

\newcommand{\balg}[1]{\begin{algorithm} \caption{#1}}
\newcommand{\ealg}{\end{algorithm}}

\newcommand{\balgc}{\begin{algorithmic}[1]}
\newcommand{\ealgc}{\end{algorithmic}}

\newcommand{\bary}{\begin{array}}
\newcommand{\eary}{\end{array}}
\newcommand{\bmx}{\begin{bmatrix}}
\newcommand{\emx}{\end{bmatrix}}
\newcommand{\bsmx}{\left[\begin{smallmatrix}}
\newcommand{\esmx}{\end{smallmatrix}\right]}
\newcommand{\bmxc}[1]{\left[\begin{array}{@{}#1@{}}}
\newcommand{\emxc}{\end{array}\right]}
\newcommand{\bcn}{\begin{center}}
\newcommand{\ecn}{\end{center}}

\newcommand{\diag}{\mathrm{diag}}

\newcommand{\Rbb}{{\mathbb{R}}}

\renewcommand{\H}{\boldsymbol{H}}

\newcommand{\X}{{\boldsymbol{X}}}

\renewcommand{\u}{\boldsymbol{u}}

\newcommand{\x}{{\boldsymbol{x}}}

\providecommand{\norm}[1]{\lVert#1\rVert}

\providecommand{\abs}[1]{\left| #1 \right|}

\newenvironment{theorem}[2][Theorem]{\begin{trivlist}
		\item[\hskip \labelsep {\bfseries #1}\hskip \labelsep {\bfseries #2.}]}{\end{trivlist}}

\usepackage{times}
\usepackage{epsfig}
\usepackage{graphicx}
\usepackage{graphics}
\usepackage{amsmath}
\usepackage{amssymb}
\newtheorem{definition}{Definition}
\makeatletter
\usepackage{subfig}
\newcommand*{\rom}[1]{\expandafter\@slowromancap\romannumeral #1@}
\makeatother

\title{Complete the Missing Half: Augmenting Aggregation Filtering with Diversification for Graph Convolutional Networks}

\author{
Sitao Luan$^{1,2,*}$, Mingde Zhao$^{1,2,*}$, Chenqing Hua$^{1,2*}$, Xiao-Wen Chang$^{1}$, Doina Precup$^{1,2,3}$\\
\{sitao.luan@mail, mingde.zhao@mail, chenqing.hua@mail, chang@cs, dprecup@cs\}.mcgill.ca\\
$^1$McGill University; $^2$Mila; $^3$DeepMind\\
$^{*}$Equal Contribution
}

\begin{document}

\maketitle

\begin{abstract}
   The core operation of current Graph Neural Networks (GNNs) is the \textit{aggregation} enabled by the graph Laplacian or message passing, which filters the neighborhood information of nodes. Though effective for various tasks, in this paper, we show that they are potentially a problematic factor underlying all GNN models for learning on certain datasets, as they force the node representations similar, making the nodes gradually lose their identity and become indistinguishable. Hence, we augment the aggregation operations with their dual, \ie{} diversification operators that make the node more distinct and preserve the identity. Such augmentation replaces the aggregation with a two-channel filtering process that, in theory, is beneficial for enriching the node representations. In practice, the proposed two-channel filters can be easily patched on existing GNN methods with diverse training strategies, including spectral and spatial (message passing) methods. In the experiments, we observe desired characteristics of the models and significant performance boost upon the baselines on $9$ node classification tasks. \footnote{See the follow-up work in \cite{luan2021heterophily,luan2022revisiting}}
\end{abstract}

\section{Introduction}
\label{sec:introduction}
As a generic data structure, graph is capable of modeling complex relations among objects in many real-world problems \cite{liao2019lanczos, monti2017geometric, defferrard2016fast}. Motivated by the success of Convolutional Neural Networks (CNNs) \cite{lecun1998gradient} on images, graph convolution \cite{wu2019survey} is defined on the graph Fourier domain and the node spatial neighborhood domain \cite{zhang2019graph}, respectively, in the form of spectral- and spatial-based methods. Based on the $2$ methodologies, different (linear) graph filters and (non-linear) deep learning techniques \cite{lecun2015deep} are combined, giving rise to Graph Neural Networks (GNNs), achieving remarkable progress \cite{bruna2014spectral,hamilton2017inductive,gilmer2017neural,kipf2016classification,velivckovic2017attention,luan2019break,luan2022revisiting}.

Most existing graph filters can be viewed as operators that aggregate node information from its direct neighbors. Different graph filters yield different spectral GNNs or spatial aggregation functions. Among them, the most commonly used is the \textit{renormalized affinity matrix} \cite{kipf2016classification}. By adding an identity matrix to the adjacency matrix, \ie{} a self-loop in the graph topology, renormalized affinity matrix is created as a low-pass (LP) filter \cite{maehara2019revisiting} mainly capturing low-frequency signals, which are locally smooth features across the whole graph \cite{wu2019simplifying}. Aggregation processes, in the form of message passing used in spatial-based methods, as in \eg{} GraphSAGE \cite{hamilton2017inductive} and GraphSAINT \cite{zeng2019graphsaint}, are also node-level LP filters which make nodes become similar to their neighbors.

The main idea of neighborhood feature aggregation is to exploit the intrinsic geometry of the data distribution: if two data points are close (or connected) to each other on the manifold, they should also be close to each other in the representation space. This assumption is usually referred to as manifold (local invariance) \cite{belkin2002laplacian,he2004locality,cai2009probabilistic,cai2010graph}, assortative mixing (assortativity) \cite{newman2003mixing}, homophily \cite{mcpherson2001birds,luan2021heterophily,lim2021large} or smoothness \cite{kalofolias2016learn,luan2022we} assumption, which plays an essential role in the development of various kinds of algorithms including dimensionality reduction \cite{belkin2002laplacian} and semi-supervised learning \cite{zhou2004learning}\footnote{In this paper, we do not distinguish the name of this assumption.}. This assumption naturally holds in many real world networks \cite{mcpherson2001birds,jiang2013assortative}, \eg{} social networks, citation networks, evolutionary biology \etc. However, in contrast to homophily, there also exists a large number of heterophily networks where individuals with diverse characteristics tend to gather in the same group \cite{rogers2010diffusion}, \eg{} dating networks \cite{zhu2020generalizing} and fraudsters in online purchasing networks \cite{pandit2007netprobe}. On these networks, there is no strong reason to impose smoothness assumption and on the contrary, non-smoothness pattern between nodes turns out to be important.




With the above in mind, in this paper we first propose a method to measure the smoothness of the input features and output labels of an attributed graph based on Dirichlet energy and graph signal energy. With the proposed method, we measure the smoothness of $9$ real world datasets, which shows that signal defined on graph is generally a mixture of smooth and non-smooth graph signals and each part plays an indispensable role. Motivated by this discovery, we argue that, to learn richer representations, the distinctive information between nodes should also be extracted. Hence, we design a two-channel filterbank (FB) \cite{ekambaram2014graph} GNN framework which use low-pass (LP) and high-pass (HP) filters together to learn the smooth and non-smooth components, respectively. FB-GNN framework can easily be plugged into spatial methods, with LP filter for aggregation operation and HP filter for diversification operation. With experiments on $9$ real world datasets, we find that the the HP channel indeed plays and important role in the representation learning, and one-channel baseline methods can gain significant performance boost after being augmented by two-channel methods.


\section{Preliminaries}
\label{sec:prelimiary_notation}
After introducing the prerequisites, in this section, we formalize the idea behind graph signal filtering. We use bold fonts for vectors (\eg{} $\bm{v}$), vector blocks (\eg{} $\bm{V}$) and matrix blocks (\eg{} $\bm{V}_i$). Suppose we have an undirected connected graph $\mathcal{G}=(\mathcal{V},\mathcal{E}, A)$ without bipartite component, where $\mathcal{V}$ is the node set with $\abs{\mathcal{V}}=N$, $\mathcal{E}$ is the edge set, $A \in \mathbb{R}^{N\times N}$ is a symmetric adjacency matrix with $A_{ij}=1$ if and only if $e_{ij} \in \mathcal{E}$ otherwise $A_{ij}=0$, $D$ is the diagonal degree matrix, \ie{} $D_{ii} = \sum_j A_{ij}$ and $\mathcal{N}_i=\{j: e_{ij} \in \mathcal{E}\}$ is the neighborhood set of node $i$. A graph signal is a vector $\bm{x} \in \mathbb{R}^N$ defined on $\mathcal{V}$, where $x_i$ is defined on the node $i$. We also have a feature matrix $\bm{X} \in \mathbb{R}^{N\times F}$ whose columns are graph signals and each node $i$ has a feature vector $\bm{X}_{i,:}$ with dimension $F$, which is the $i$-th row of $\bm{X}$. 

\subsection{Graph Laplacian and Affinity Matrix} \label{sec:laplacian_affinity_matrix}
The (Combinatorial) graph Laplacian is defined as $L = D - A$, which is a Symmetric Positive Semi-Definite (SPSD) matrix\cite{chung1997spectral}. Its eigendecomposition gives $L=U\Lambda U^T$, where the columns of $U\in \Rbb^{N\times N}$ are orthonormal eigenvectors, namely the \textit{graph Fourier basis}, $\Lambda = \diag(\lambda_1, \ldots, \lambda_N)$ with $\lambda_1 \leq \cdots \leq \lambda_N$ and these eigenvalues are also called \textit{frequencies}. The graph Fourier transform of the graph signal $\x$ is defined as $\bm{x}_\mathcal{F} = U^{-1} \bm{x} = U^{T} \bm{x} = [\u_1^T\x, \ldots, \u_N^T\x]^T$, where $\bm{u}_i^T \bm{x}$ is the component of $\bm{x}$ in the direction of $\bm{u_i}$. 

A smaller $\lambda_i$ indicates a smoother basis function $\bm{u_i}$ defined on $\mathcal{G}$ \cite{dakovic2019local}, which means any two elements of $\bm{u_i}$ corresponding to two connected nodes will have more similar values. This is because finding the eigenvalues and eigenvectors of graph Laplacian is actually solving a series of conditioned minimization problems relevant to the smoothness of the function defined on $\mathcal{G}$.  

Some variants of graph Laplacians are commonly used in practice, \eg{} the symmetric normalized Laplacian $L_{\text{sym}} = D^{-1/2} L D^{-1/2} = I-D^{-1/2} A D^{-1/2}$, the random walk normalized Laplacian $L_{\text{rw}} = D^{-1} L = I - D^{-1} A$. $L_{\text{rw}}$ and $L_{\text{sym}}$ share the same eigenvalues, which are inside $[0,2)$, and their corresponding eigenvectors satisfy $\bm{u}_{\text{rw}}^i = D^{-1/2} \bm{u}_{\text{sym}}^i$. 

The affinity (transition) matrix derived from $L_{\text{rw}}$ is defined as $A_\text{rw} = I - L_\text{rw} = D^{-1} A$ and its eigenvalues $\lambda_i(A_\text{rw}) = 1-\lambda_i(L_\text{rw}) \in (-1,1]$. Similarly, $A_\text{sym} = I-L_\text{sym} = D^{-1/2} A D^{-1/2}$ is an affinity matrix as well. Renormalized affinity matrix is introduced in \cite{kipf2016classification} and defined as $\hat{A}_{\text{rw}} = \tilde{D}^{-1} \tilde{A}$, where $\tilde{A} \equiv A+I, \tilde{D} \equiv D+I$ and $\lambda(\hat{A}_\text{rw}) \in (-1,1]$. It essentially defines a random walk matrix on $\mathcal{G}$ with a self-loop added to each node in $\mathcal{V}$ and is widely used in GCN as follows,
\begin{equation}
    \label{eq:gcn_original}
   \bm{Y} = \text{softmax} (\hat{A}_\text{rw} \; \text{ReLU} (\hat{A}_\text{rw} \bm{X} W_0 ) \; W_1 )
\end{equation}
where $W_0 \in \Rbb^{F\times F_1}$ and $W_1 \in \Rbb^{F_1\times O}$ are parameter matrices. $\hat{L}_\text{rw}$ can be defined as $I-\hat{A}_\text{rw}$. $\hat{A}_\text{sym} \equiv \tilde{D}^{-1/2} \tilde{A} \tilde{D}^{-1/2}$ can also be applied in GCN and $\hat{L}_{\text{sym}} = I - \hat{A}_\text{sym}$. Specifically, the nature of transition matrix makes $\hat{A}_{\text{rw}}$ behave as a mean aggregator $(\hat{A}_{\text{rw}} \bm{x})_i = \sum_{j\in\{\mathcal{N}_i \cup i\}} {x}_j/(D_{ii}+1)$ which is applied in \cite{hamilton2017inductive} and is important to bridge the gap between spatial- and spectral-based graph convolution methods.

\subsection{Measure of Smoothness and (Dirichlet) Energy}

Dirichlet Energy is often used to measure how variable a function is \cite{evans1998partial} and for signal defined on graph, it can measure the global smoothness of the signal \cite{shuman2013emerging,bronstein2016geometric, smith2018graph} and is defined as follows.

\begin{definition}(Dirichlet Energy) The Dirichlet energy of vector block $\bm{X}$ and column vector $\bm{x}$ defined on $\mathcal{G}$ are separately defined as 
\begin{equation} \label{def:dirichlet_energy}
    E_S^\mathcal{G}(\bm{X}) = tr(\bm{X}^T L \bm{X}), \ \ E_S^\mathcal{G}(\bm{x}) = \bm{x}^T L \bm{x}
\end{equation}
\end{definition}

Note that $E_S^\mathcal{G}$ is always non-negative since $L$ is SPSD. The graph signal energy is defined as follows.

\begin{definition}(Graph Signal Energy \cite{gavili2017shift, stankovic2018reduced})
The signal energy of vector block $\bm{X}$ and column vector $\bm{x}$ defined on undirected graph $\mathcal{G}$ are separately defined as  
\begin{equation} \label{def:energy}
    E^\mathcal{G}(\bm{X}) = tr(\bm{X}^T \bm{X}) , \ \ E^\mathcal{G}(\bm{x}) = \bm{x}^T \bm{x} 
\end{equation}
\end{definition}
The signal energy represents the amount of contents in a graph signal and we will draw the correlation between $E_S^\mathcal{G}$ and $E^\mathcal{G}$ and explain how they can be used to measure the smoothness and non-smoothness of a graph (block) signal.

Take column vector $\bm{x}$ for example, $E_S^\mathcal{G}(\bm{x})$ can be written as, 
$$\bm{x}^T L \bm{x} = \sum\limits_i \lambda_i (u_i^T \bm{x})^T u_i^T \bm{x} = \sum\limits_i \lambda_i \norm{u_i^T \bm{x}}_2^2 $$
The frequency $\lambda_i$ before $\norm{u_i^T \bm{x}}_2^2$ can be considered as a scalar weight and $E_S^\mathcal{G}(\bm{x})$ focuses on measuring the component of $\bm{x}$ in the direction of non-smooth $\bm{u_i}$, who has a large weight $\lambda_i$. A small $E_S^\mathcal{G}(\bm{x})$ means $\bm{x}$ does not contain much non-smooth components. $E^\mathcal{G}(\bm{x})$ can be written as
$$\bm{x}^T  \bm{x} = \sum\limits_i  (u_i^T \bm{x})^T u_i^T \bm{x} = \sum\limits_i  \norm{u_i^T \bm{x}}_2^2$$

Signal $\bm{x}$ can be decomposed into smooth and non-smooth components, and the amount the non-smooth component can be measured by
$$E_{NS}^\mathcal{G}(\bm{x}) = E^\mathcal{G}(\bm{x}) -  E_{S}^\mathcal{G}(\bm{x}) = \bm{x}^T (I-L) \bm{x} = \sum\limits_i (1-\lambda_i) \norm{u_i^T \bm{x}}_2 $$

Note that $E_{NS}^\mathcal{G}(\bm{x})$ can be negative  and a small $E_{NS}^\mathcal{G}(\bm{x})$ indicates that $\bm{x}$ is highly non-smooth.

Upon the above analysis, we define $S(\bm{x})$ to measure the smoothness of a signal as follows
\begin{equation}
\label{eq:def_smoothness}
    S(\bm{x}) = \frac{E_{S}^\mathcal{G}(\bm{x})}{E^\mathcal{G}(\bm{x})}, \ \ S(\bm{X}) = \frac{E_{S}^\mathcal{G}(\bm{X})}{E^\mathcal{G}(\bm{X})}
\end{equation}
Graph signal with a small $S$-value means it is a smooth function define on $\mathcal{G}$. $S$ can be different depends on the Laplacian we use to train GNN and $S$ can be larger than 1. In this paper, we use $L_{\text{sym}}$ and $\hat{L}_{\text{sym}}$ to measure the smoothness of input features $\bm{X}$ and labels $\bm{y}$ for different GNNs.

\section{Filterbanks in GNNs: From One Channel to Two}

\label{sec:graph_filterbank_networks}
In this section, we state why it is necessary to switch to the two-channel filtering process from only one-channel. Then, we propose the filterbank-GNN framework which can learn a mixture of smooth and non-smooth graph signals.

\subsection{Motivation}
\label{sec:motivation}

\begin{table*}[htbp]
  \centering
  \caption{Dataset Overview: Network Characteristics and $S$-values measured by $L_{\text{sym}}$}
  \resizebox{\textwidth}{!}{
    \begin{tabular}{p{5.9em}llllllllll}
    \toprule
    \toprule
    \multicolumn{2}{c}{datasets} & \multicolumn{1}{c}{Cornell} & \multicolumn{1}{c}{Wisconsin} & \multicolumn{1}{c}{Texas} & \multicolumn{1}{c}{Actor} & \multicolumn{1}{c}{Chameleon} & \multicolumn{1}{c}{Squirrel} & \multicolumn{1}{c}{Cora} & \multicolumn{1}{c}{CiteSeer} & \multicolumn{1}{c}{PubMed} \\
    \midrule
    \multicolumn{1}{c}{\multirow{4}[2]{*}{Network Info}} & \multicolumn{1}{c}{\#nodes} & \multicolumn{1}{c}{183} & \multicolumn{1}{c}{251} & \multicolumn{1}{c}{183} & \multicolumn{1}{c}{7600} & \multicolumn{1}{c}{2277} & \multicolumn{1}{c}{5201} & \multicolumn{1}{c}{2708} & \multicolumn{1}{c}{3327} & \multicolumn{1}{c}{19717} \\
    \multicolumn{1}{c}{} & \multicolumn{1}{c}{\#edges} & \multicolumn{1}{c}{295} & \multicolumn{1}{c}{499} & \multicolumn{1}{c}{309} & \multicolumn{1}{c}{33544} & \multicolumn{1}{c}{36101} & \multicolumn{1}{c}{217073} & \multicolumn{1}{c}{5429} & \multicolumn{1}{c}{4732} & \multicolumn{1}{c}{44338} \\
    \multicolumn{1}{c}{} & \multicolumn{1}{c}{\#features} & \multicolumn{1}{c}{1703} & \multicolumn{1}{c}{1703} & \multicolumn{1}{c}{1703} & \multicolumn{1}{c}{931} & \multicolumn{1}{c}{2325} & \multicolumn{1}{c}{2089} & \multicolumn{1}{c}{1433} & \multicolumn{1}{c}{3703} & \multicolumn{1}{c}{500} \\
    \multicolumn{1}{c}{} & \multicolumn{1}{c}{\#classes} & \multicolumn{1}{c}{5} & \multicolumn{1}{c}{5} & \multicolumn{1}{c}{5} & \multicolumn{1}{c}{5} & \multicolumn{1}{c}{5} & \multicolumn{1}{c}{5} & \multicolumn{1}{c}{7} & \multicolumn{1}{c}{6} & \multicolumn{1}{c}{3} \\
    \midrule
    \multicolumn{1}{c}{\multirow{3}[2]{*}{S-values}} & \multicolumn{1}{c}{input feature} & \multicolumn{1}{c}{0.904} & \multicolumn{1}{c}{0.873} & \multicolumn{1}{c}{0.854} & \multicolumn{1}{c}{0.901} & \multicolumn{1}{c}{0.99} & \multicolumn{1}{c}{0.987} & \multicolumn{1}{c}{0.862} & \multicolumn{1}{c}{0.799} & \multicolumn{1}{c}{0.832} \\
    \multicolumn{1}{c}{} & \multicolumn{1}{c}{label} & \multicolumn{1}{c}{0.883} & \multicolumn{1}{c}{0.877} & \multicolumn{1}{c}{0.909} & \multicolumn{1}{c}{0.836} & \multicolumn{1}{c}{0.747} & \multicolumn{1}{c}{0.782} & \multicolumn{1}{c}{0.288} & \multicolumn{1}{c}{0.35} & \multicolumn{1}{c}{0.501} \\
    \multicolumn{1}{c}{} & \multicolumn{1}{c}{diff (label - feature)} & \multicolumn{1}{c}{\cellcolor[rgb]{ .867,  .922,  .969}-0.021} & \multicolumn{1}{c}{\cellcolor[rgb]{ .988,  .894,  .839}0.004} & \multicolumn{1}{c}{\cellcolor[rgb]{ .988,  .894,  .839}0.055} & \multicolumn{1}{c}{\cellcolor[rgb]{ .867,  .922,  .969}-0.065} & \multicolumn{1}{c}{\cellcolor[rgb]{ .867,  .922,  .969}-0.243} & \multicolumn{1}{c}{\cellcolor[rgb]{ .867,  .922,  .969}-0.205} & \multicolumn{1}{c}{\cellcolor[rgb]{ .867,  .922,  .969}-0.574} & \multicolumn{1}{c}{\cellcolor[rgb]{ .867,  .922,  .969}-0.449} & \multicolumn{1}{c}{\cellcolor[rgb]{ .867,  .922,  .969}-0.331} \\
    \midrule
    \midrule
    \multicolumn{11}{p{64.065em}}{We use blue and red shades to demonstrate the relation between label and feature: the label of the blue shaded datasets is smoother than its feature and red datasets are less smooth.} \\
    \end{tabular}%
    }
  \label{tab:dataoverview}%
\end{table*}%

We measure the smoothness of $9$ frequently used benchmark datasets and present the results with the network characteristics for each task in Table \ref{tab:dataoverview}). It shows that the input features and ground truth labels of different datasets are all mixtures of smooth and non-smooth graph signals but in different proportions. Besides, it illustrates that different tasks have different demands of learning how to smoothen the input signals. For example, in \textit{Cora}, \textit{CiteSeer} and \textit{PubMed}, the ground truth labels are much smoother than the input features, such pattern motivates us to learn how to smoothen the input signals; while in \textit{Wisconsin} and \textit{Texas}, the labels are less smooth than the input features, thus there is no reason that we still learn how to smoothen the input signals. Therefore, to accommodate different situations, we  propose that we should learn both smooth and non-smooth components of the input features adaptively instead of merely extracting the smooth part. This motivates us to use filterbanks (LP and HP filters) to filter the signals in GNNs.

\paragraph{LP, HP Graph Filters and Filter Banks}
The multiplication of $L$ and $\bm{x}$ acts as a filtering operation over $\bm{x}$, adjusting the scale of the components of $\bm{x}$ in frequency domain. To see this, consider
\begin{equation} 
    \bm{x} = \sum_i \bm{u}_i  \bm{u}_i ^T \bm{x} , \ \ L\bm{x} = \sum_i \lambda_i  \bm{u}_i  \bm{u}_i ^T \bm{x} 
    \label{eq:explicit_conv_L}
\end{equation}
The projection $\u_i\u_i^T\x$ corresponding to a large $|\lambda_i|$ will be amplified, while the one corresponding to a small $|\lambda_i|$ will be suppressed. More specifically, a graph filter that filters out smooth (non-smooth) components is called HP (LP) filter. Generally, the Laplacian matrices ($L_{sym}$, $L_{rw}$, $\hat{L}_{sym}$, $\hat{L}_{rw}$) can be regarded as HP filters \cite{ekambaram2014graph} and affinity matrices ($A_\text{sym}$, $A_\text{rw}$, $\hat{A}_\text{sym}$, $\hat{A}_\text{rw}$) can be treated as LP filters \cite{maehara2019revisiting}. In general, we denote HP and LP filters as $L_\text{HP}$ and $L_\text{LP}$ respectively.

On the node level, left multiplying HP and LP filters on $\bm{x}$ can be understood as diversification and aggregation operations, respectively. For example, if we implement $L_{rw}$ and ${A}_\text{rw}$ on the $i$-th node, we have
\begin{align} \label{eq:node_level_L_P}
    &(L_\text{rw} \bm{x})_i = \sum_{j \in \mathcal{N}_i } \frac{1}{D_{ii}} (x_i-x_j),\ \ (A_\text{rw} \bm{x})_i = \sum_{j \in \mathcal{N}_i } \frac{1}{D_{ii}} x_j
\end{align}
Intuitively, HP filters depict the differences between one node and its neighbors; While LP filters focus on the similarity within a neighborhood, from which we can obtain missing or ``hidden'' features of one node. We believe that these two conjugate components are both indispensable to portray a node.

Mathematically, multiplying with LP filter (aggregation) is a linear projection, which will project the features to a fixed subspace. We will lose the expressive power by only using LP filter, and the missing half is the HP component of the learned signals, as $L_\text{LP} + L_\text{HP} = I$, which satisfies the perfect reconstruction property \cite{ekambaram2014graph}.

The two-channel linear filterbank which contains a set of filters $L_\text{LP}$ and $L_\text{HP}$ is widely used in graph signal processing \cite{ekambaram2013critically, ekambaram2014graph}, but are rarely used in graph neural networks. Inspired by this technique, we propose the two-channel filterbank GNNs in section \ref{sec:FB-GNN}, which can extract both smooth and non-smooth components from input features. 

\subsection{Filter Bank assisted GNNs (FB-GNNs)}\label{sec:FB-GNN}

\begin{figure}[htbp]
\centering
\includegraphics[width=0.5\textwidth]{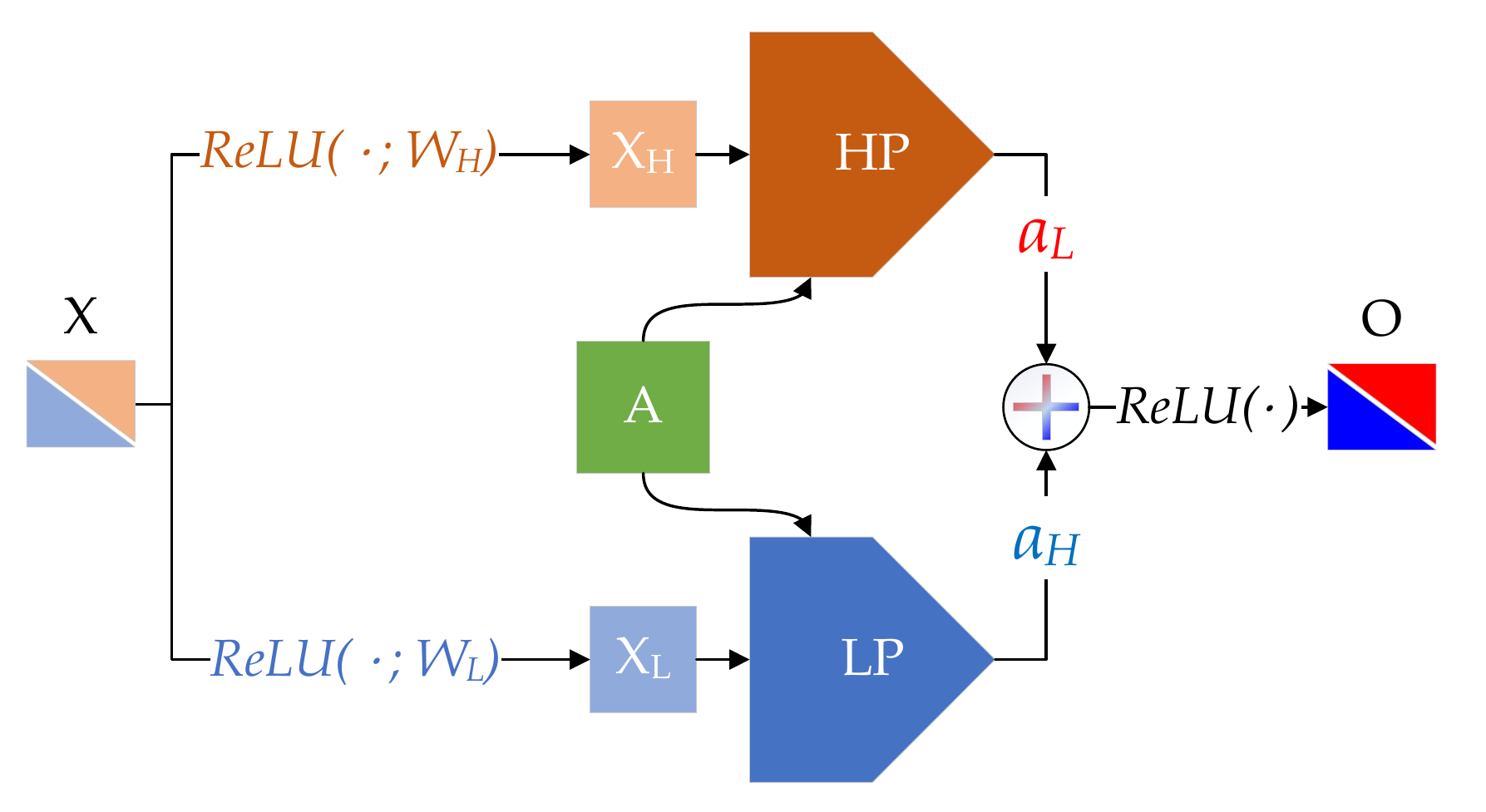}
\caption{Two-Channel Learning: Information needed for HP and LP filters, $X_H$ and $X_L$, are separately extracted from the input signal $X$ by nonlinear transformations. After being filtered by HP and LP, which are both derived upon adjacency matrix $A$, the filtered signals are again recombined adaptively to form the output $O$.}
\label{fig:two_pass}
\end{figure}


\paragraph{Spectral-based FB-GNNs}
We use previously defined $L_\text{LP}$ and $L_\text{HP}$ (see section \ref{sec:motivation}) to construct the two-channel FB-GNNs as follows (learning framework is provided in figure \ref{fig:two_pass})
\begin{align}
\label{eq:FB-GNN_spectral}
{\H}^{l}_L &= L_\text{LP} \text{ReLU}({\H^{l-1}} W^{l-1}_L), \ {{\H}^{l}_H} = L_\text{HP} \text{ReLU}({\H^{l-1}} W^{l-1}_H) \\
{\H^{l}} & = \text{ReLU}\left(\alpha_L^l \cdot {\H}^{l}_L + \alpha_H^l \cdot {\H}^{l}_H \right), \ \ l=1,\dots,n \nonumber
\end{align}
where ${\H^0}=\bm{X}$; $W_L^{l-1},W_H^{l-1} \in \mathbb{R}^{F_{l-1} \times F_l}$ are learnable parameter matrices for the non-linear feature extractor focusing on disentangling the smooth and non-smooth information from input ${\H^{l-1}}$, separately; $\alpha_L^l, \alpha_H^l \in [0,1]$ are learnable scalar parameters which can learn the relative importance of ${\H}^{l}_L$ and ${\H}^{l}_H$ and keep a balance between them; $l$ is the layer number and suppose the FB-GNN has $n$ layers. In this way, the hidden output ${\H}^{l}$ is able to learn a mixture of smooth and non-smooth signals.

\paragraph{Spatial-based FB-GNNs} Inspired by \eqref{eq:node_level_L_P} and \eqref{eq:FB-GNN_spectral}, the two-channel spatial-based method can be implemented by designing different aggregator (LP filter) and diversification operator (HP filter) as follows,
\begin{align} \label{eq:FB-GNN_spatial_hidden_output}
 (\bm{\hat{h}}_{i}^{l})_L & = \text{ReLU}({W}_L^{l-1} \bm{h}_{i}^{l-1}), \ \ (\bm{\hat{h}}_{i}^{l})_H = \text{ReLU}({W}_H^{l-1} \bm{h}_{i}^{l-1}) \nonumber \\ 
(\bm{{h}}_{i}^{l})_L & =  \sum_{j \in \{\mathcal{N}_i \cup i\} } \bm{w}_{ij} \left( (\bm{\hat{h}}_{i}^{l})_L + (\bm{\hat{h}}_{j}^{l})_L \right), \ 
 (\bm{{h}}_{i}^{l})_H = \sum_{j \in \{\mathcal{N}_i \cup i\} } \bm{w}_{ij} \left((\bm{\hat{h}}_{i}^{l})_H  - (\bm{\hat{h}}_{j}^{l})_H \right), \\
 \bm{{h}}_{i}^{l} = & \text{ReLU}\left(\alpha^l_L \cdot (\bm{{h}}_{i}^{l})_L + \alpha^l_H \cdot (\bm{{h}}_{i}^{l})_H\right), \ , i \in \mathcal{V}, \ l = 1,\dots,n \nonumber
\end{align}
where ${W}_L^{l-1}, {W}_H^{l-1} \in \mathbb{R}^{F_l \times F_{l-1}}$ are learnable parameter matrices to extract LP and HP features for two channels; $\bm{w}_{ij}$ is the connection weight between node $i$ and node $j$ derived from adjacency matrix, it can be a fixed value or a learnable attention coefficient such as \cite{velivckovic2017attention}; $\alpha^l_L, \alpha^l_H \in [0,1]$ are learnable scalar parameters; $l$ is the layer number.

\paragraph{Computational Cost: Parameters and Runtime}
The spectral two-channel learning introduces additionally one GCN operation and one weighted sum (with negligible costs introduced with non-linearity and weighted sum before output); For spatial methods, similarly, the two-channel learning introduces one additional node-wise subtraction and one additional weighted sum for training on each pair of nodes. Thus, the computational cost and the number of parameters are approximately doubled;

For runtime, overlooking the minor overhead of synchronization, the computations introduced by the additional pass are naturally parallelizable with the original pass (for their independently associated parameters) both in the forward and backward passes. Therefore, no significant additional computational time will be incurred on modern GPU architectures.

\section{Related Works}
\paragraph{Dirichlet energy} Dirichlet energy (more generally in $p$-Dirichlet form) is usually used as a regularizer or objective function to impose local neighborhood smoothness in various machine learning tasks, \eg{} spectral clustering \cite{belkin2002laplacian}, image processing \cite{elmoataz2008nonlocal,bougleux2009local,zheng2010graph}, non-negative matrix factorization \cite{cai2010graph}, matrix completion, principal component analysis (PCA), semi-supervised learning \cite{zhu2003semi,zhu2005harmonic,belkin2006manifold}. It has different names in different literature, \eg{} Laplacian regularizer \cite{zheng2010graph}, manifold regularizer \cite{cai2010graph}, quadratic energy function \cite{zhu2003semi}, \etc.
The definition of smoothness derived from Dirichlet energy in \eqref{eq:def_smoothness} can be considered as a continuous relaxed form of normalized cut (ratio cut) \cite{hagen1992new,stankovic2019graph}, which is closely related to graph partition problems. Instead of using Dirichlet energy as a part of loss function during training process, we point out that combining with graph signal energy, it can be used to measure the smoothness of the input features and output labels for a given learning task. With this, the necessity of learning the non-smoothness component can be confirmed.

\paragraph{Measuring Smoothness} The authors of \cite{pei2020geom} propose a node homophily to measure the smoothness of ground truth labels of dataset as follows,
\begin{equation*}
\frac{1}{|\mathcal{V}|} \sum_{v \in \mathcal{V}} \frac{ \# v \text { 's neighbors who have the same label as } v}{\# v \text{ 's neighbors }}
\end{equation*}
\cite{zhu2020generalizing} proposes edge homophily ratio, which is the fraction of edges that the connected nodes share the same label (\ie{}, intra-class edges). Both of these methods do not provide an extension definition on block vector. Thus, they fail to measure the smoothness of the input features and cannot be used to compare the difference of smoothness between the input features and labels. \cite{zhao2019pairnorm} proposes row-diff and col-diff to measure the average of all pairwise distances between the node features and the average of pairwise distances between columns of the representation matrix. But quantifying pairwise distance is inconsistent with the definition of smoothness introduced in section \ref{sec:introduction} which focuses on measuring the distance between connected nodes. \cite{luan2022revisiting} studies homophily from post-aggregation node similarity perspective. \cite{luan2022we} uses statistical hypothesis testing to detect the effect the edge bias.

\paragraph{On Addressing Heterophily} 
Geom-GCN \cite{pei2020geom} uses a geometric aggregation scheme and a bi-level aggregator to capture the information of structural neighborhoods, which can be distant nodes. These can efficiently take use of the geometric relationships defined in the latent space. H$_2$GCN \cite{zhu2020generalizing} designs ego- and neighbor-embedding separation, aggregation of higher-order neighborhoods, and combination of intermediate representations to generalize the limitation of existing GNNs beyond homophily setting. Non-local GNNs \cite{liu2020non} propose a simple and effective non-local aggregation framework with an efficient attention-guided sorting for GNNs.  CPGNN \cite{zhu2020graph} models label correlations through a compatibility matrix, which is beneficial for heterophilic graphs, and propagates a prior belief estimation into the GNN by using the compatibility matrix.  FAGCN \cite{bo2021beyond} learns edge-level aggregation weights as GAT \cite{velivckovic2017attention} but allows the weights to be negative, which enables the network to capture high-frequency components in the graph signals. GPRGNN \cite{chien2021adaptive} uses learnable weights that can be both positive and negative for feature propagation. This allows GPRGNN to adapt to heterophilic  graphs and  to handle both high- and low-frequency parts of the graph signals. BernNet \cite{he2021bernnet} designs a scheme to learn arbitrary graph spectral filters with Bernstein polynomial to address heterophily.

The aforementioned works design various tricks, trying to take use of multi-hop neighborhood information and capture long-range dependencies with the belief that heterophily problem could be alleviated with the help of the distant nodes. Although these methods show some promising results, the effectiveness is limited and do not jump out of the scope of neighborhood aggregation. In this paper, We target directly its cause, handling heterophily problem by seeking the distinctiveness between nodes with an additional channel to learn the non-smoothness components.
\section{Experiments}

\label{sec:experiments}

In this section, we first validate whether the two-channel filtering and learning procedure lead to better representation learning when patched on popular shallow GNN baselines\footnote{Source code submitted within supplementary materials and to be published after the review.}: GraphSAINT \cite{zeng2019graphsaint}, GraphSAGE \cite{hamilton2017inductive}, Graph Attention Network (GAT) \cite{velivckovic2017attention}, GCN \cite{kipf2016classification}, Geom-GCN-P (-S and -I) \cite{pei2020geom} and Graph Wavelet Neural Network (GWNN) \cite{xu2019wavelet}. Deeper GNNs are shown to have the potentials of mitigating the heterophily problem by extracting multi-hop neighborhood information. For them, we test the two-channel framework on two state-of-the-art methods GCN\rom{2} and GCN\rom{2}* \cite{chen2020simple}, with varied model depths. After these, we validate the effectiveness of each proposed component with a detailed ablation test.

The experiments are conducted in the form of node classification\footnote{See Appendix \ref{appendix:graph_classification} for experimental results on graph classification tasks.} under supervised learning setting and performed on $9$ datasets including \textit{Cornell}, \textit{Wisconsin}, \textit{Texas}, \textit{Actor}, \textit{Chameleon}, \textit{Squirrel}, \textit{Cora}, \textit{CiteSeer}, and \textit{PubMed} (details to be found in the appendix). Their rough characteristics are shown in Table \ref{tab:dataoverview}.

\subsection{Experimental Setup}

In supervised learning of shallow GNNs, we keep the same training configurations for GraphSAINT and FB-GraphSAINT on the $9$ datasets, which are the random walk sampler with length 2 (RW) setting\footnote{The name ``random walk sampler with length 2 setting'' is what the authors used in their paper and they use the name $PPI$-large-2 in their code.} in GraphSAINT \cite{zeng2019graphsaint};  for GWNN \cite{xu2019wavelet} and FB-GWNN, we use the same hyperparameters $s=1.0, t=10^{-4}$ on the $9$ datasets\footnote{This set of hyperparameters is the same as that of the original paper when training GWNN on Cora.}. Other GNNs and their two-channel variants are under the same experiment settings as \cite{hamilton2017inductive} and \cite{pei2020geom}. We use $\hat{A}_\text{sym}$ as low-pass spectral filter and $\hat{L}_\text{sym}$ as high-pass spectral filter \footnote{We will discuss other filters in Appendix \ref{appendix:discussion_filters}}.

For supervised learning on deep GNNs, GCN\rom{2}, GCN\rom{2}*, FB-GCN\rom{2}, and FB-GCN\rom{2}* use $\lambda=1.5$ and $\alpha=0.2$ on $Actor$ and $Squirrel$. Other deep models use the same training configurations as GCN\rom{2} \cite{chen2020simple} on the remaining 7 datasets.

For all experiments, we use the same 48\%/ 32\%/ 20\% splits for training,
validation and testing as in \cite{pei2020geom}. We report the average performance of all models on the test sets over $10$  splits\footnote{We obtain the performance of GAT, GCN, and GEOM-GCN-P(-S and -I) directly from \cite{pei2020geom}; and we reproduce other baseline models and implement all the filterbank GNNs.}. We tune the learning rate in \{0.01, 0.05, 0.1\}, weight decay in \{0, 5e-6, 1e-5, 5e-5, 1e-4, 5e-4, 1e-3, 5e-3, 1e-2\}, and dropout in \{0, 0.1, 0.2, \dots, 0.9\}.

\begin{table*}[htbp]
  \centering
  \tiny
  \caption{Supervised Learning of Shallow GNNs}
  \resizebox{\textwidth}{!}{
    \begin{tabular}{c|ccccccccc}
    \toprule
    \toprule
    Models/Datasets & Cornell & Wisconsin & Texas & Actor & Chameleon & Squirrel & Cora & CiteSeer & PubMed \\
    \midrule
    Diff of $S$-values & -0.021   & 0.004  & 0.055   & -0.065  & -0.243 & -0.205 & -0.574  & -0.449  & -0.330  \\
    \midrule
    \multicolumn{10}{c}{Spatial Methods(\%)} \\
    GraphSAINT & 70.27 & 71.35 & 72.97 & 17.89 & 43.86 & 33.27 & 84.69 & 73.2  & 89.42 \\
    FB-GraphSAINT & 78.38(8.11) & 80(8.65) & 75.68(2.71) & 19.08(1.19) & 46.05(2.19) & 36.06(2.79) & 87.5(2.81) & 74.76(1.56) & 89.88(0.46) \\
    GraphSAGE & 54.05 & 66    & 56.76 & 14.67 & 40.13 & 24.14 & 80.64 & 72.36 & 85.49 \\
    FB-GraphSAGE & 63.14(9.09) & 70(4) & 58.05(1.29) & 23.27(8.6) & 39.74(-0.39) & 24.6(0.46) & 83.7(3.06) & 72.58(0.22) & 86.31(0.82) \\
    \midrule
    \multicolumn{10}{c}{Spectral Methods(\%)} \\
    GAT   & 54.32 & 49.41 & 58.38 & 28.45 & 42.93 & 30.03 & 86.37 & 74.32 & 87.62 \\
    FB-GAT & 64.86(10.54) & 60.78(11.37) & 64.86(6.48) & 30.66(2.21) & 47.37(4.44) & 31.8(1.77) & \textbf{88.73(2.36)} & 77.12(2.8) & 88.16(0.54) \\
    GCN   & 52.7  & 45.88 & 52.16 & 26.86 & 28.18 & 23.96 & 85.77 & 73.68 & 88.13 \\
    FB-GCN & 62.16(9.46) & 56.86(10.98) & 62.16(10.00) & 31.21(4.35) & 32.89(4.71) & 24.73(0.77) & 85.92(0.15) & 75.24(1.56) & 88.54(0.41) \\
    Geom-GCN-P & 60.81 & 64.12 & 67.57 & \textbf{31.63} & 60.9  & 38.14 & 84.93 & 75.14 & 88.09 \\
    FB-Geom-GCN-P & 64.86(4.05) & 72.55(8.43) & 70.27(2.70) & 31.02(-0.61) & \textbf{67.20(6.30)} & \textbf{49.66(11.52)} & 85.17(0.24) & 76.23(1.09) & 88.25(0.16) \\
    Geom-GCN-S & 55.68 & 56.67 & 59.73 & 30.3  & 59.96 & 36.24 & 85.27 & 74.71 & 84.75 \\
    FB-Geom-GCN-S & 56.54(0.86) & 56.94(0.27) & 62.16(2.43) & 31.25(0.95) & 61.49(1.53) & 37.27(1.03) & 85.43(0.16) & 75.21(0.5) & 85.88(1.13) \\
    Geom-GCN-I & 56.76 & 58.24 & 57.58 & 29.09 & 60.31 & 33.32 & 85.19 & \textbf{77.99} & 90.05 \\
    FB-Geom-GCN-I & 57.38(0.62) & 60.68(2.44) & 62.21(4.63) & 31.45(2.36) & 60.76(0.45) & 35.27(1.95) & 85.45(0.26) & 77.69(-0.3) & \textbf{90.48(0.43)} \\
    GWNN  & 70.67 & 72.22 & 69.44 & 20.92 & 33.63 & 29.13 & 84.49 & 72.47 & 83.6 \\
    FB-GWNN & \textbf{80.11(9.44)} & \textbf{84.67(12.45)} & \textbf{77.78(8.34)} & 22.24(1.32) & 37.36(3.73) & 30.6(1.47) & 85.6(1.11) & 72.83(0.36) & 85.92(2.32) \\
    \midrule
     Baseline Average & 59.41 & 60.49 & 62.32 & 26.2  & 46.46 & 31.44 & 85.15 & 74.33 & 87.3 \\
    FB-Baseline Average & 65.93(6.52) & 67.81(\textbf{7.32}) & 66.65(4.33) & 27.52(1.32) & 49.11(2.65) & 33.75(2.31) & 85.94(1.27) & 75.21(0.97) & 87.93(0.63) \\
    \bottomrule
    \bottomrule
    \multicolumn{10}{p{66em}}{The results are averaged from $10$ independent runs. The (values) represent the difference of performance brought by patching FB.} \\
    \end{tabular}}
  \label{tab:supervised_learning_shallow_gnns}%
\end{table*}%

\begin{table*}[htbp]
  \centering
  \caption{Statistics of Datasets (measured by $\hat{L}_{\text{sym}}$ instead of $L_{\text{sym}}$) and Comparison of the Output Smoothness}
  \resizebox{\textwidth}{!}{
    \begin{tabular}{ccccccccccc}
    \toprule
    \toprule
    \multicolumn{2}{c}{datasets} & Cornell & Wisconsin & Texas & Actor & Chameleon & Squirrel & Cora & CiteSeer & PubMed \\
    \midrule
    \multirow{5}[4]{*}{S-values} & input feature & 0.172 & 0.385 & 0.205 & 0.567 & 0.831 & 0.87  & 0.617 & 0.515 & 0.529 \\
          & \textit{label} & \textit{0.139} & \textit{0.328} & \textit{0.301} & \textit{0.511} & \textit{0.638} & \textit{0.681} & \textit{0.188} & \textit{0.209} & \textit{0.272} \\
          & diff (label - feature) & \cellcolor[rgb]{ .867,  .922,  .969}-0.033 & \cellcolor[rgb]{ .867,  .922,  .969}-0.057 & \cellcolor[rgb]{ .988,  .894,  .839}0.096 & \cellcolor[rgb]{ .867,  .922,  .969}-0.056 & \cellcolor[rgb]{ .867,  .922,  .969}-0.193 & \cellcolor[rgb]{ .867,  .922,  .969}-0.189 & \cellcolor[rgb]{ .867,  .922,  .969}-0.429 & \cellcolor[rgb]{ .867,  .922,  .969}-0.306 & \cellcolor[rgb]{ .867,  .922,  .969}-0.257 \\
\cmidrule{2-11}          & \textit{GCN output} & \textit{0.037 (0.102)} & \textit{0.124 (0.204)} & \textit{0.139 (0.162)} & \textit{0.397 (0.114)} & \textit{0.595 (0.043)} & \textit{0.578 (0.103)} & \textit{0.156 (0.032)} & \textit{0.112 (0.097)} & \textit{0.234 (0.038)} \\
          & \textit{FB-GCN output} & \textit{\textbf{0.099 (0.040)}} & \textit{\textbf{0.269 (0.059)}} & \textit{\textbf{0.201 (0.100)}} & \textit{\textbf{0.531 (0.020)}} & \textbf{\textit{0.655 (0.017)}} & \textit{\textbf{0.683 (0.002)}} & \textit{\textbf{0.172 (0.016)}} & \textit{\textbf{0.148 (0.061)}} & \textit{\textbf{0.247 (0.025)}} \\
    \bottomrule
    \bottomrule
    \multicolumn{11}{p{77.335em}}{These results are obtained from $10$ independent runs. The stds are negligible so they are not presented (mostly $<0.002$).This table shows how FB- patched baseline could better reconstruct the label smoothness, \ie{} we want the S-value of the output to be closer to that of the labels. The (values) stand for the absolute difference between the S-values of the output of the methods and those of the ground truths. Better reconstruction between the two methods on each task is marked \textbf{bold}.} \\

    \end{tabular}%
    }
  \label{tab:Dataset_stats_training_smoothness}%
\end{table*}%

\begin{figure*}[htbp]
\centering
{
\subfloat[Low-Pass Channel]{
\captionsetup{justification = centering}
\includegraphics[width=0.32\textwidth]{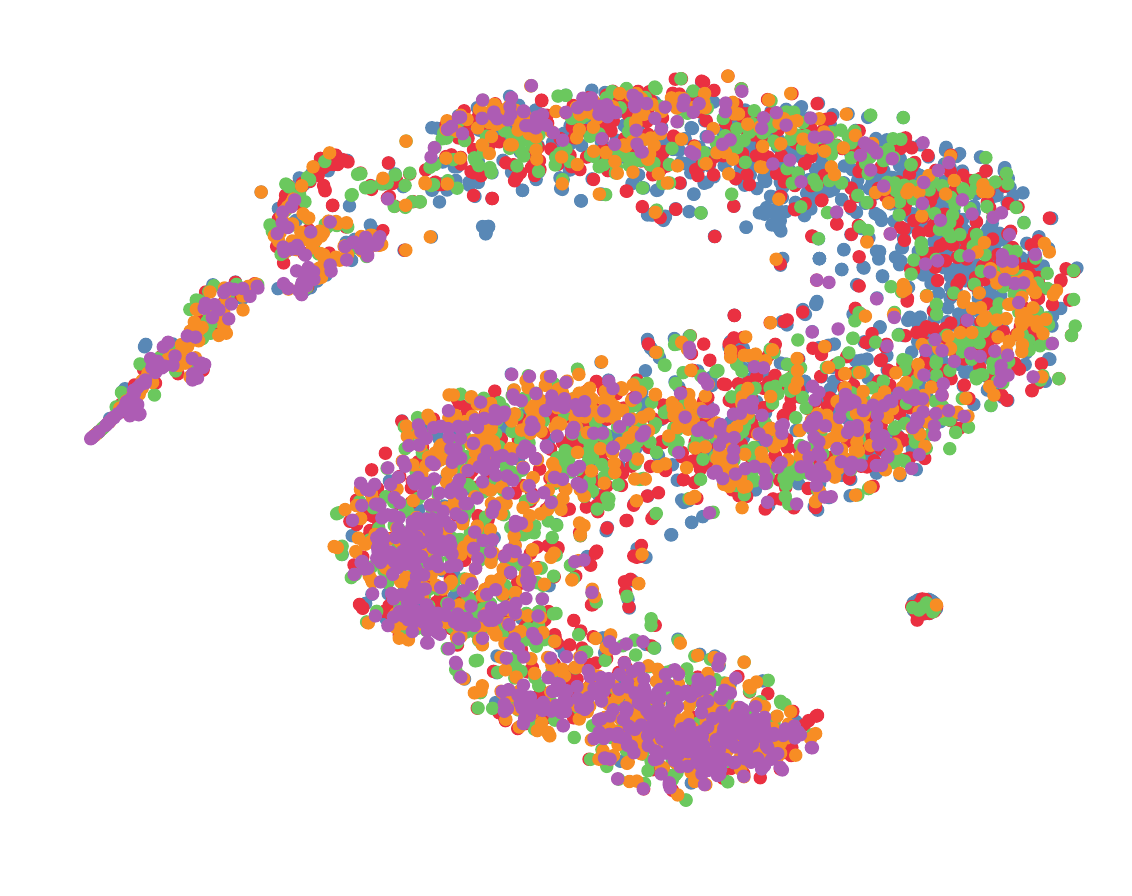}}
\subfloat[High-Pass Channel]{
\captionsetup{justification = centering}
\includegraphics[width=0.32\textwidth]{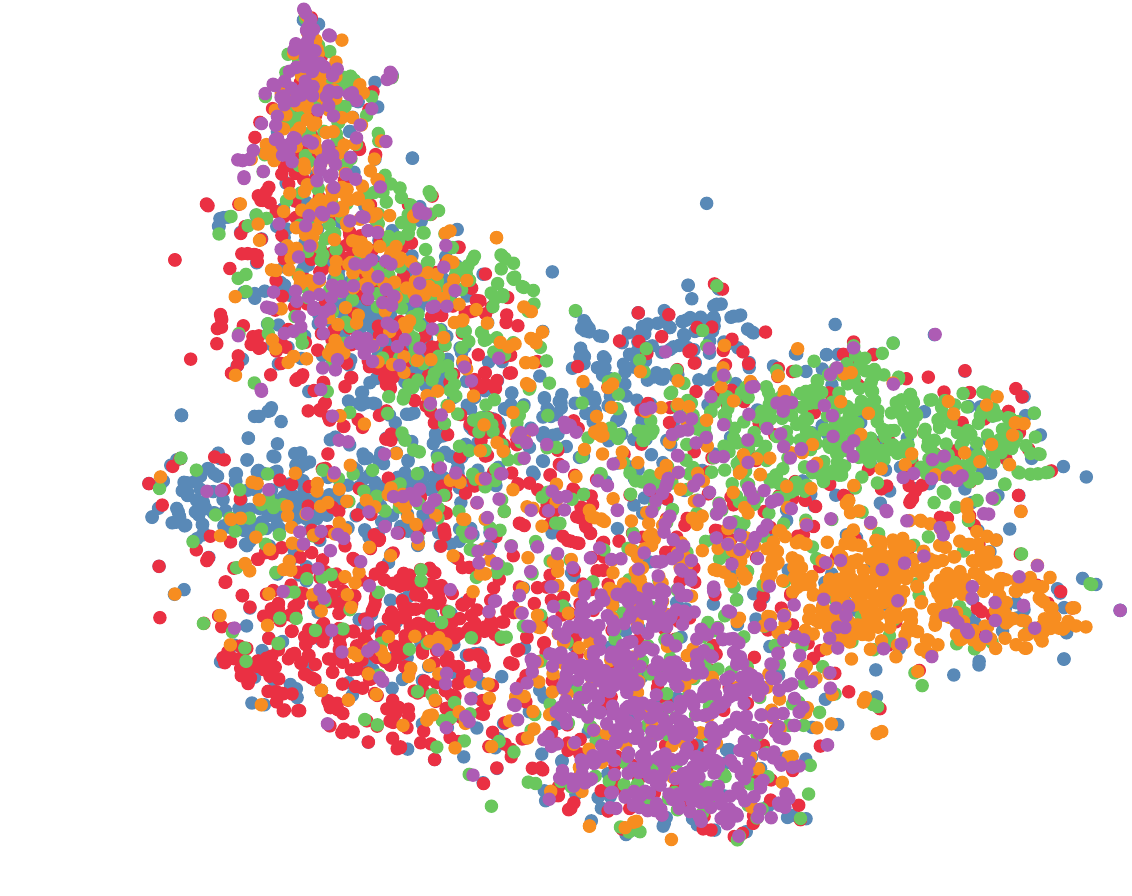}}
\subfloat[Two Channels Combined]{
\captionsetup{justification = centering}
\includegraphics[width=0.32\textwidth]{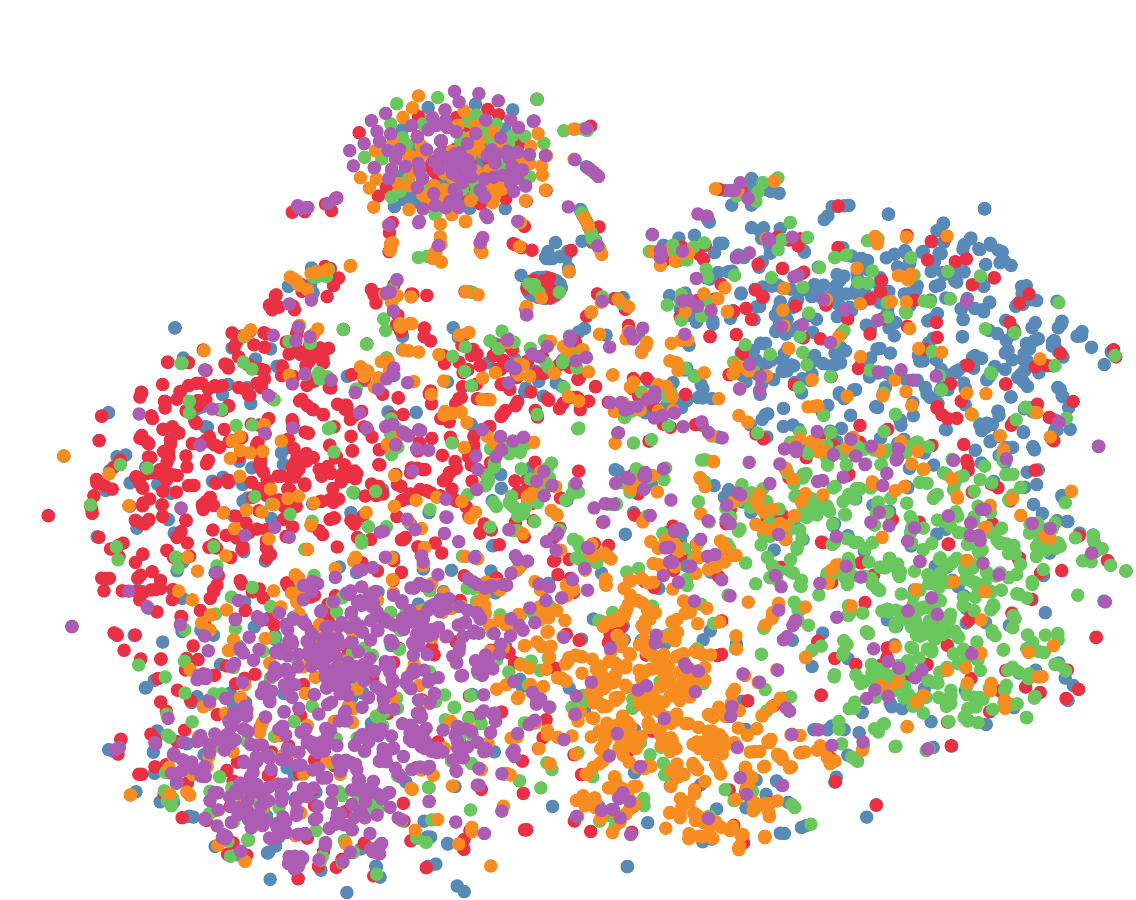}}
}
\caption{$t$-SNE Visualization of the Learned Node Embeddings in Different Channels of FB-GCN for Squirrel Dataset.} %
\label{fig:tsne_squirrel}
\end{figure*}

\subsection{Supervised Learning of Shallow GNNs}
In Table \ref{tab:supervised_learning_shallow_gnns}, we summarize the mean accuracy of shallow baseline GNNs and their filterbank versions. The best performance is highlighted. Also, we record the performance differences between baselines and the two-channel augmented methods in the brackets. For better intuitive understanding of the effectiveness of incorporating the high-pass filter, we present in Figure \ref{fig:tsne_squirrel} the t-SNE visualization of the learned embedding in different channels of FB-GCN for Squirrel dataset. Those of the other datasets will be provided in the appendix \ref{appendix:more_tsne}.

From the results we can see that, our proposed method generally boost the performance of almost all cases, especially when the labels are not much smoother than the input features indicated, considering the $S$-values in Table \ref{tab:Dataset_stats_training_smoothness}. 

Table \ref{tab:Dataset_stats_training_smoothness} shows that the $S$-values of FB-GCN outputs are closer to the $S$-values of the ground truth labels (see the absolute differences in the bracket) compared with those of GCN outputs. This indicates that FB-GCN is able to learn better representations which can reconstruct both the smooth and non-smooth part of the ground truth labels. Note that we measure the smoothness by $\hat{L}_\text{sym}$ instead of ${L}_\text{sym}$, because GCN is train with renormalized affinity matrix $\hat{A}_\text{sym}$.

\begin{table*}[htbp]
  \centering
  \caption{Supervised Learning of Deep Multi-scale GNNs}
  \resizebox{\textwidth}{!}{
    \begin{tabular}{c|ccccccccc}
    \toprule
    \toprule
    Models\textbackslash{}Datasets & Cornell & Wisconsin & Texas & Actor & Chameleon & Squirrel & Cora & CiteSeer & PubMed \\
    \midrule
    GCN\rom{2}-8 & 70.54 & 73.88 & 71.08 & 33.7  & 60.61 & 37.49 & 85.69 & 75.54 & 88.62 \\
    FB-GCN\rom{2}-8 & 75.95(5.41) & 82.35(8.47) & 74.59(3.51) & 35.37(1.67) & 60.43(-0.18) & 39.69(2.2) & 86.04(0.35) & 75.51(-0.03) & 89.97(1.35) \\
    GCN\rom{2}-16 & 74.86 & 74.12 & 69.46 & 33.62 & 55.48 & 35.98 & 87.3  & 76.54 & 88.28 \\
    FB-GCN\rom{2}-16 & \textbf{77.57(2.71)} & 82.55(8.43) & 77.03(7.57) & 35.12(1.5) & 56.78(1.3) & 39.38(3.4) & 87.5(0.2) & 76.67(0.13) & 89.39(1.11) \\
    GCN\rom{2}-32 & 72.7  & 70.2  & 69.46 & 31.61 & 53.71 & 35.92 & 88.13 & 76.08 & 87.89 \\
    FB-GCN\rom{2}-32 & 72.81(0.11) & 78.63(8.43) & 80.27(10.81) & 32.99(1.38) & 54.98(1.27) & 36.81(0.89) & 88.33(0.2) & 76.49(0.41) & 89.1(1.21) \\
    GCN\rom{2}-64 & 71.89 & 68.84 & 66.49 & 28.76 & 54.14 & 36.1  & \textbf{88.49} & 77.08 & 89.57 \\
    FB-GCN\rom{2}-64 & 76.49(4.6) & 76.27(7.43) & 76.22(9.73) & 29.57(0.81) & 54.39(0.25) & 36.79(0.69) & 87.92(-0.57) & 77(-0.08) & 89.65(0.08) \\
    GCN\rom{2}*-8 & 72.97 & 78.82 & 72.7  & 34.89 & 62.48 & 40.72 & 86.14 & 75.06 & 89.7 \\
    FB-GCN\rom{2}*-8 & 76.76(3.79) & \textbf{82.94(4.12)} & 78.11(5.41) & \textbf{35.87(0.98)} & \textbf{65.11(2.63)} & 41.19(0.47) & 86.94(0.8) & 76.32(1.26) & 90.2(0.5) \\
    GCN\rom{2}*-16 & 76.49 & 81.57 & 75.41 & 34.18 & 58.86 & 39.88 & 87.46 & 75.8  & 86.69 \\
    FB-GCN\rom{2}*-16 & 76.95(0.46) & 82.39(0.82) & 76.76(1.35) & 35.4(1.22) & 59.98(1.12) & 40.08(0.2) & 87.48(0.02) & 76.43(0.63) & 89.95(3.26) \\
    GCN\rom{2}*-32 & 74.32 & 77.06 & 77.84 & 33.78 & 56.27 & 37.69 & 88.35 & 76.55 & 89.37 \\
    FB-GCN\rom{2}*-32 & 74.51(0.19) & 80.78(3.72) & \textbf{84.86(7.02)} & 34.73(0.95) & 57.65(1.38) & \textbf{41.24(3.55)} & 88.16(-0.19) & 76.89(0.34) & 89.92(0.55) \\
    GCN\rom{2}*-64 & 72.43 & 73.53 & 75.41 & 32.72 & 53.82 & 36.83 & 88.01 & \textbf{77.13} & \textbf{90.3} \\
    FB-GCN\rom{2}*-64 & 75.84(3.41) & 81.57(8.04) & 80.54(5.13) & 34.89(2.17) & 57.52(3.7) & 39.81(2.98) & 87.44(-0.57) & 77.03(-0.1) & 89.98(-0.32) \\
    \midrule
    Baseline Average & 73.28 & 74.75 & 72.23 & 32.91 & 56.92 & 37.58 & 87.45 & 76.22 & 88.8 \\
    FB-Baseline Average & 75.86(2.58) & 80.94(6.19) & 78.55(\textbf{{6.32}}) & 34.24(1.33) & 58.36(1.44) & 39.37(1.79) & 87.48(0.03) & 76.54(0.32) & 89.77(0.97) \\
    \bottomrule
    \bottomrule
    \multicolumn{10}{p{66em}}{The results are averaged from $10$ independent runs. The (values) represent the difference of performance brought by patching FB.} \\
    \end{tabular}}
  \label{tab:supervised_learning_deep_gnns}%
\end{table*}%

\subsection{Supervised Learning of Deep GNNs}
In this subsection, we build deep multi-hop filterbank models based on the architecture of GCN\rom{2} and GCN\rom{2}* \cite{chen2020simple} to see if the two-channel method is capable to assist deep GNN models. We report mean accuracy, highlight best performing depth, and record performance difference in brackets in Table \ref{tab:supervised_learning_deep_gnns}. In general, FB-GCN\rom{2} and FB-GCN\rom{2}* achieve better results than the unpatched GCN\rom{2} and GCN\rom{2}* at different depths, especially on \textit{Wisconsin} and \textit{Texas}, where the non-smooth part of representations are desirable.

\begin{table*}[htbp]
  \centering
  \setlength{\tabcolsep}{3pt}
  \caption{Ablation Results: Accuracy (\%)}
    \begin{tabular}{cccccccc}
    \toprule
    \toprule
    \multirow{2}[2]{*}{\#Channels} & \multirow{2}[2]{*}{Transformation} & \multicolumn{2}{c}{Cora} & \multicolumn{2}{c}{Cornell} &  \multicolumn{2}{c}{Texas} \\
          &       & Mean & Std   & Mean & Std   & Mean & Std \\
    \midrule
    1     & linear & \cellcolor[rgb]{ .973,  .412,  .42}83.92 & 1.0   & \cellcolor[rgb]{ .973,  .412,  .42}64.86 & 2.2   &  \cellcolor[rgb]{ .973,  .412,  .42}70.60 & 1.9 \\
    1     & nonlinear & \cellcolor[rgb]{ .992,  .831,  .498}84.69 & 0.5   & \cellcolor[rgb]{ .988,  .749,  .482}70.27 & 0.8   &  \cellcolor[rgb]{ .992,  .82,  .498}72.97 & 0.8 \\
    2     & linear & \cellcolor[rgb]{ .965,  .914,  .518}85.02 & 2.1   & \cellcolor[rgb]{ .698,  .835,  .502}75.64 & 2.0   &  \cellcolor[rgb]{ .831,  .875,  .51}74.15 & 2.1 \\
    \textit{\textbf{2}} & \textit{\textbf{nonlinear}} & \cellcolor[rgb]{ .388,  .745,  .482}\textit{\textbf{87.50}} & \textit{1.6} & \cellcolor[rgb]{ .388,  .745,  .482}\textit{\textbf{78.38}} & \textit{1.5} &  \cellcolor[rgb]{ .388,  .745,  .482}\textit{\textbf{75.68}} & \textit{1.0} \\
    \bottomrule
    \bottomrule
    \multicolumn{8}{p{26em}}{\small Color indicators are added to differentiate the performance of each test case: the greener the better, the redder the worse.} \\
    \end{tabular}%
  \label{tab:ablations}%
\end{table*}%

\subsection{Ablation Tests}\label{sec:ablation}
In this subsection, we perform ablation tests by using FB-GraphSAINT \cite{zeng2019graphsaint} on \textit{Cora}, \textit{Cornell} and \textit{Texas} accordingly, which are the three principally different datasets measured by $\hat{L}_{\text{sym}}$  in Table \ref{tab:Dataset_stats_training_smoothness}
. The ablation tests would examine the effectiveness of each proposed component. The results are summarized in Table \ref{tab:ablations}. The results show that both the non-linear feature extractor and two-channel filtering architecture are able to help capture richer information under different $S$-value distributions of input features and output labels.(See more ablation results in Appendix \ref{appendix:ablation_ppi}) 

Moreover, to emphasize the importance of HP component, we also test the learnable coefficients $\alpha_L$ and $\alpha_H$ for two components on FB-GraphSAINT over the 9 datasets at the validation stage (see Table \ref{tab:ablation_coefficients} in Appendix \ref{appendix:ablation_coefficients}). For most of the tasks, neither the coefficients for LP nor HP are negligible. Among $6$ out of $9$ tasks, the learned coefficients for the HP components are even greater, this indicates the necessity of the HP components in graph representation learning.

In addition, from the averaged real-time change of the learned coefficients $\alpha_L$ and $\alpha_H$ in the output layer of FB-GraphSAINT on \textit{Cora}, \textit{Cornell} and  \textit{Texas} during training (see Figure \ref{fig:alphas_saint} in Appendix \ref{appendix:change_of_coefficients}), we can see that $\alpha_L$ and $\alpha_H$ will converge to a pair of values that explains how the smooth and non-smooth features will be mixed. The fact that the ratio $\alpha_H/\alpha_L$ is close to $1$ again shows that the importance of the non-smooth part in constructing the output signal. More specifically, the importance (red line) is higher when the demand of non-smooth outputs (diff values) is higher. See more ablation study of GCN on PPI in Appendix \ref{appendix:ablation_ppi}.

\section{Conclusion}
\label{sec:discussion}

This paper recognizes the role of high-frequency information in graph representation learning. The proposed HP filter completes the spectrum of graph filters and yield significantly better representations on several empirical tasks. The importance of the non-smooth component in graph signals can be revealed by the new defined S-value.

\clearpage
{
\bibliographystyle{abbrv}
\bibliography{references}
}

\clearpage

\appendix

\section{Dataset Descriptions}

For node classification. there are $4$ main categories: 

\textit{Cora}, \textit{CiteSeer}, and \textit{PubMed} are $3$ benchmark datasets \cite{sen:aimag08} in the category of \textit{Citation network}. Such networks use nodes to represent papers and edges to denote citations. Node features are the bag-of-words representation and node labels are classified into different academic topics.

\textit{Cornell}, \textit{Texas}, and \textit{Wisconsin} belong to the webpage dataset \textit{WebKB} \cite{pei2020geom} created by Carnegie Mellon University. Each node represents a web page, and the edges are hyperlinks between nodes. Node features are the bag-of-words representation and node labels are in five classes.

\textit{Chameleon} and \textit{Squirrel} are twi page-to-page networks in the \textit{Wikipedia network} \cite{rozemberczki2019multiscale}. Nodes represent web pages and edges show mutual links between pages. Node features are informative nouns in the Wikipedia pages and nodes are classified into $5$ groups based on monthly views.

\textit{Actor} refers to the \textit{Actor co-occurrence network}. A node correspond to an actor, and an edge exists if two actors occur on the same Wikipedia page. Node features correspond to some keywords in the Wikipedia pages, and nodes are categorized into five classes of words of actors actor's Wikipedia.

\section{Graph Classification}
\label{appendix:graph_classification}

For the graph classification tasks, we compare the patched methods FB-GIN-$0$ and FB-GIN-$\epsilon$ against the baselines GIN-$0$ and GIN-$\epsilon$, with the same experiment setting as \cite{xu2018powerful}. The results (accuracy and standard deviation) are provided in Table \ref{tab:graph_classification}.

\begin{table*}[htbp]
  \centering
  \small
  \caption{Results and Hyperparameters of Graph Classification}
    \begin{tabular}{cccccccccc}
    \toprule
    \toprule
    Task  & Method & lr    & weight decay & gamma & width & batch size & dropout & concat & Acc \\
    \midrule
    \multirow{4}[2]{*}{MUTAG} & GIN-0 &       &       &       &       &       &       &       & 89.4	 \\
          & FB-GIN-0 & 0.036104 & 0.0001034 & -0.48239 & 128   & 32    & 0.75127 & 0     & \textbf{91.4035} \\
          & GIN-eps &       &       &       &       &       &       &       & 89	 \\
          & FB-GIN-eps & 0.003584 & 0.011275 & ~     & 64    & 32    & 0.75859 & 0     & \textbf{94.74} \\
    \midrule
    \multirow{4}[2]{*}{PROTEINS} & GIN-0 &       &       &       &       &       &       &       & 76.2	 \\
          & FB-GIN-0 & 0.047597 & 0.00042991 & 1.269 & 8     & 32    & 0.064654 & 1     & \textbf{80.784} \\
          & GIN-eps &       &       &       &       &       &       &       & 75.9	 \\
          & FB-GIN-eps & 0.006173 & 0.02751 & ~     & 8     & 128   & 0.26964 & 0     & \textbf{79.4375} \\
    \midrule
    \multirow{4}[2]{*}{PTC} & GIN-0 &       &       &       &       &       &       &       & 64.6 \\
          & FB-GIN-0 & 0.0083924 & 0.0059482 & 0.72171 & 16    & 32    & 0.082378 & 0     & \textbf{68.578} \\
          & GIN-eps &       &       &       &       &       &       &       & 63.7 \\
          & FB-GIN-eps & 0.049951 & 0.00029225 & ~     & 16    & 128   & 0.30306 & 0     & \textbf{71.429} \\
    \midrule
    \multirow{4}[2]{*}{NCI1} & GIN-0 &       &       &       &       &       &       &       & 	82.7 \\
          & FB-GIN-0 & 0.00039327 & 0.01014 & 0.95777 & 128   & 128   & 0.01526 & 1     & \textbf{84.428} \\
          & GIN-eps &       &       &       &       &       &       &       & 	82.7 \\
          & FB-GIN-eps & 9.94E-05 & 0.0083156 & ~     & 128   & 128   & 0.70224 & 1     & \textbf{84.123} \\
    \midrule
    \multirow{4}[2]{*}{IMDB-B} & GIN-0 &       &       &       &       &       &       &       & 75.1	 \\
          & FB-GIN-0 & 0.010815 & 0.00024241 & 0.83553 & 128   & 128   & 0.97456 & 1     & \textbf{83} \\
          & GIN-eps &       &       &       &       &       &       &       & 74.3 \\
          & FB-GIN-eps & 0.015596 & 0.0047105 & ~     & 32    & 32    & 0.80636 & 1     & \textbf{78.111} \\
    \midrule
    \multirow{4}[2]{*}{IMDB-M} & GIN-0 &       &       &       &       &       &       &       & 52.3	 \\
          & FB-GIN-0 & 0.00067325 & 0.0042346 & 1.4691 & 64    & 128   & 0.80828 & 1     & \textbf{53.467} \\
          & GIN-eps &       &       &       &       &       &       &       & 	52.1 \\
          & FB-GIN-eps & 0.00061908 & 0.037266 & ~     & 64    & 128   & 0.92727 & 1     & \textbf{53.259} \\
    \midrule
    \multirow{4}[2]{*}{RDT-B} & GIN-0 &       &       &       &       &       &       &       & 92.4 \\
          & FB-GIN-0 & 0.01262 & 0.047278 & -0.41963 & 8     & 128   & 0.48795 & 0     & \textbf{94} \\
          & GIN-eps &       &       &       &       &       &       &       & 	92.2 \\
          & FB-GIN-eps & 0.0068918 & 0.016003 & ~     & 128   & 128   & 0.4131 & 0     & \textbf{93} \\
    \midrule
    \multirow{4}[2]{*}{RDT-M5K} & GIN-0 &       &       &       &       &       &       &       & 	57.5	 \\
          & FB-GIN-0 & 0.0011204 & 0.017434 & -0.02748 & 8     & 128   & 0.35465 & 1     & \textbf{65.6} \\
          & GIN-eps &       &       &       &       &       &       &       & 	57	 \\
          & FB-GIN-eps & 0.0026491 & 0.0492 & ~     & 8     & 128   & 0.55127 & 1     & \textbf{68.4} \\
    \midrule
    \multirow{4}[2]{*}{COLLAB} & GIN-0 &       &       &       &       &       &       &       & 80.2 \\
          & FB-GIN-0 & 2.88E-04 & 0.047982 & -0.3438 & 128   & 128   & 0.66614 & 1     & \textbf{86.3} \\
          & GIN-eps &       &       &       &       &       &       &       & 80.1 \\
          & FB-GIN-eps & 0.00019472 & 0.00031991 & ~     & 128   & 128   & 0.1088 & 1     & \textbf{85} \\
    \bottomrule
    \end{tabular}%
  \label{tab:graph_classification}%
\end{table*}%

\section{More Ablation Tests}
\subsection{Ablation Coefficients}
\label{appendix:ablation_coefficients}
\begin{table*}[htbp]
  \centering
  \small
  \caption{$\alpha_L$ and $\alpha_H$ in the Output Layer of FB-GraphSAINT}
  \resizebox{\textwidth}{!}{
    \begin{tabular}{cccccccccc}
    \toprule
    \toprule
          & Cornell & Wisconsin & Texas & Actor  & Chameleon & Squirrel & Cora & CiteSeer & PubMed \\
    \midrule
    $\alpha_L$ & 0.436 & 0.441 & 0.57 & 0.54 & 0.701 & 0.675  & 0.509 & 0.514 & 0.473 \\
    $\alpha_H$ & \textbf{0.45} & \textbf{0.499} & \textbf{0.6} & \textbf{0.557} & \textbf{0.713} & 0.65 & 0.464 & 0.503 & \textbf{0.478} \\
    \midrule
    $\alpha_H / \alpha_L$  & \textbf{1.032} & \textbf{1.132} & \textbf{1.053} & \textbf{1.031} & \textbf{1.017} & 0.963 & 0.912 & 0.979 & \textbf{1.011} \\
    \bottomrule
    \bottomrule
    \multicolumn{10}{p{50.33em}}{The results are averaged from $10$ independent runs. If the ratio is higher than 1.0, then the high frequency signals are more important. The higher the ratio, the more important HP filter is.} \\

    \end{tabular}%
    }
  \label{tab:ablation_coefficients}%
\end{table*}%
\subsection{Real Time Change of Coefficients}
\label{appendix:change_of_coefficients}
\begin{figure*}[htbp]
\centering
{
\subfloat[Cora (diff = -0.429)]{
\captionsetup{justification = centering}
\includegraphics[width=0.32\textwidth]{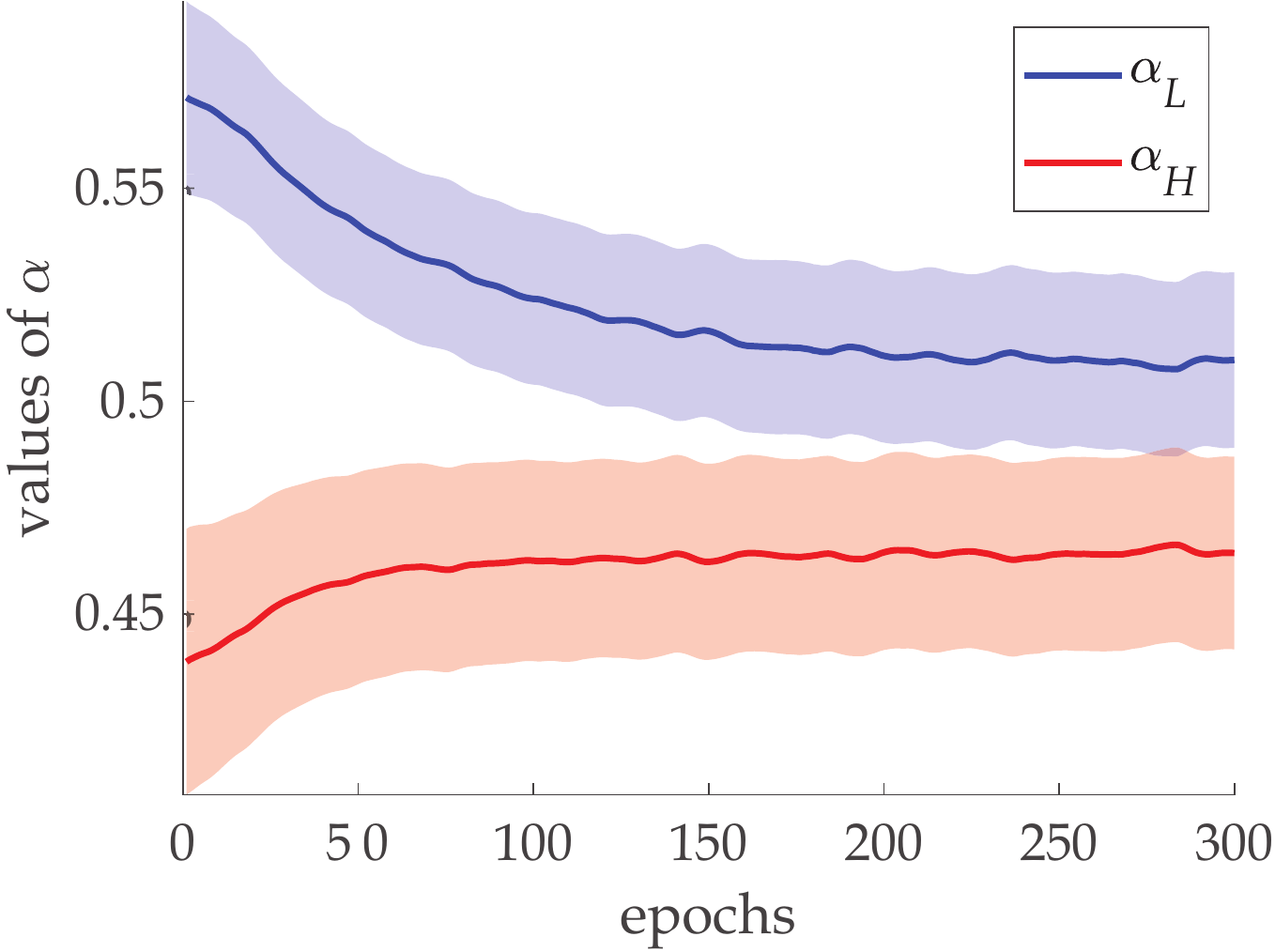}}
\subfloat[Cornell (diff = -0.033)]{
\captionsetup{justification = centering}
\includegraphics[width=0.32\textwidth]{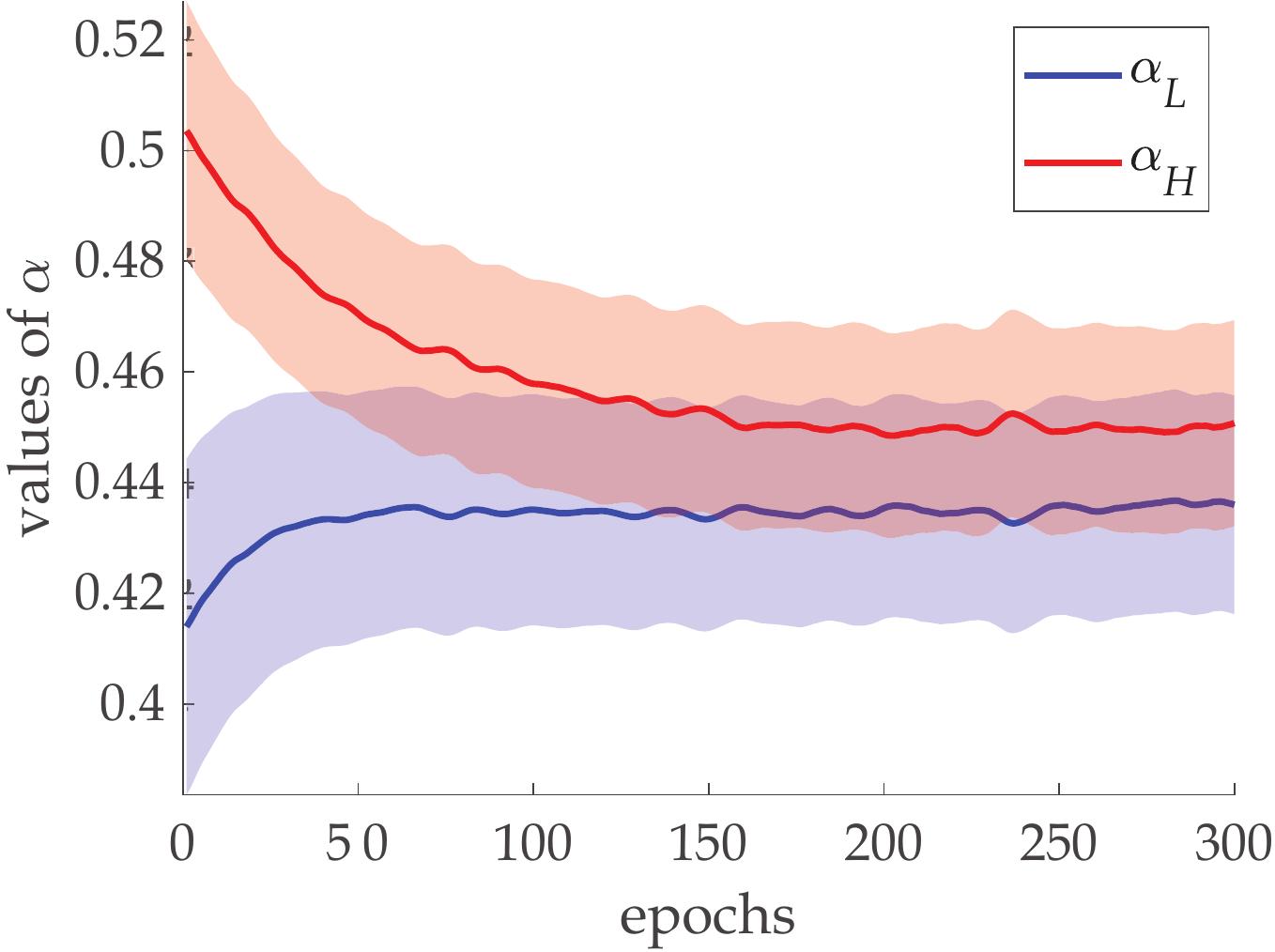}}
\subfloat[Texas (diff = 0.096)]{
\captionsetup{justification = centering}
\includegraphics[width=0.32\textwidth]{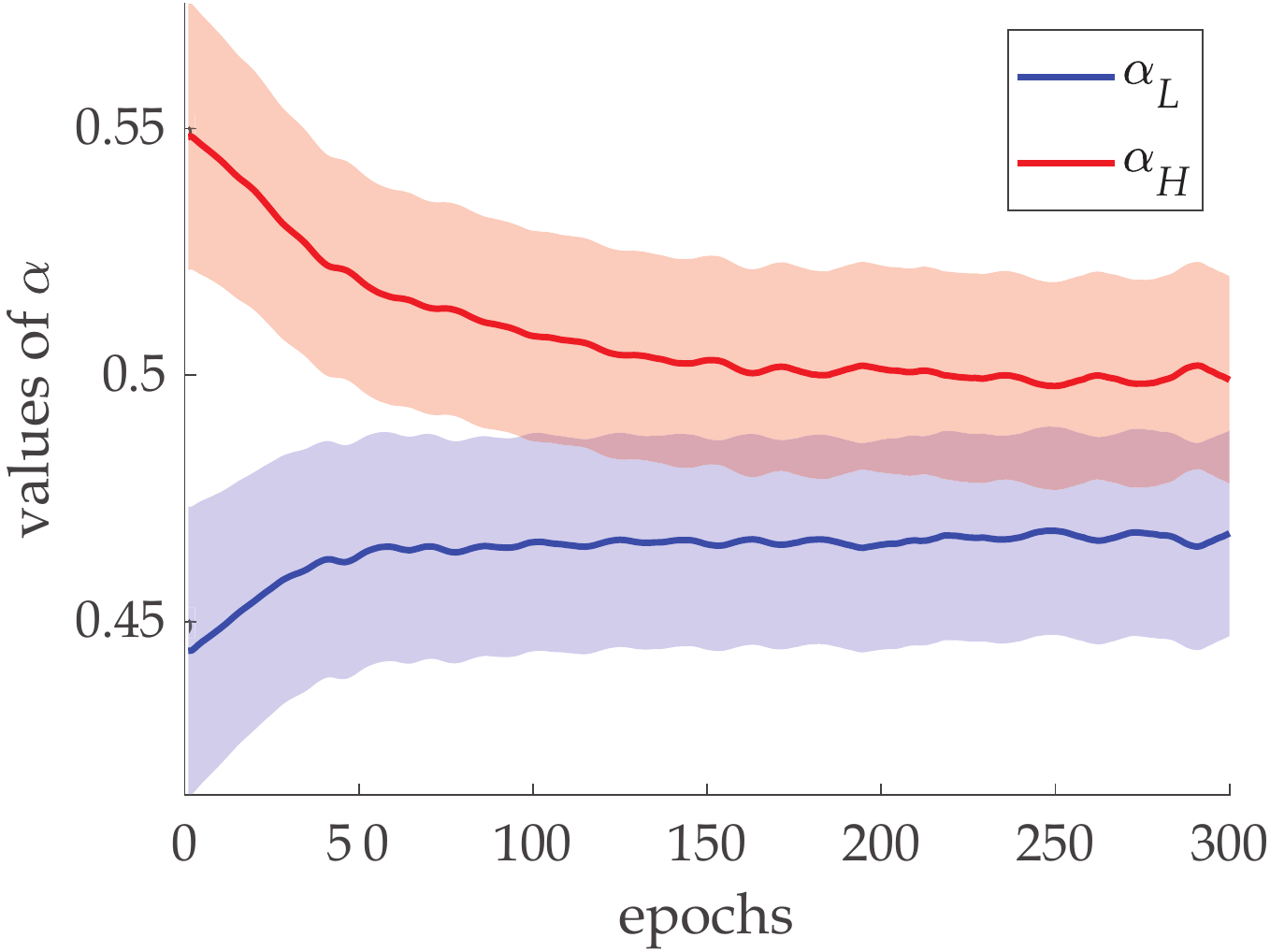}}
}
\caption{$\alpha_L$ and $\alpha_H$ in the output layer of FB-GraphSAINT trained on Cora, Cornell and Texas. The mean curves and the std bands are obtained over $20$ independent runs. See diff values in table .} 
\label{fig:alphas_saint}
\end{figure*}

\subsection{Ablation Tests on PPI}
\label{appendix:ablation_ppi}
\begin{table}[htbp]
  \centering
  \caption{Ablation Tests on PPI}
    \begin{tabular}{c|c|c|c}
    \toprule
    \toprule
    \#Channels & Transformation & F1-score & Std \\
    \midrule
    1     & linear & 59.4  & 0.8 \\
    \midrule
    1     & nonlinear & 69.5  & 0.3 \\
    \midrule
    2     & linear & 71.8  & 0.6 \\
    \midrule
    2     & nonlinear & 73.9  & 0.4 \\
    \bottomrule
    \bottomrule
    \end{tabular}%
  \label{tab:ablation_ppi}%
\end{table}%

\section{t-SNE Visualization}
\label{appendix:more_tsne}
\clearpage
\begin{figure*}[htbp]
\centering
{
\subfloat[Actor: LP Channel]{
\captionsetup{justification = centering}
\includegraphics[width=0.3\textwidth]{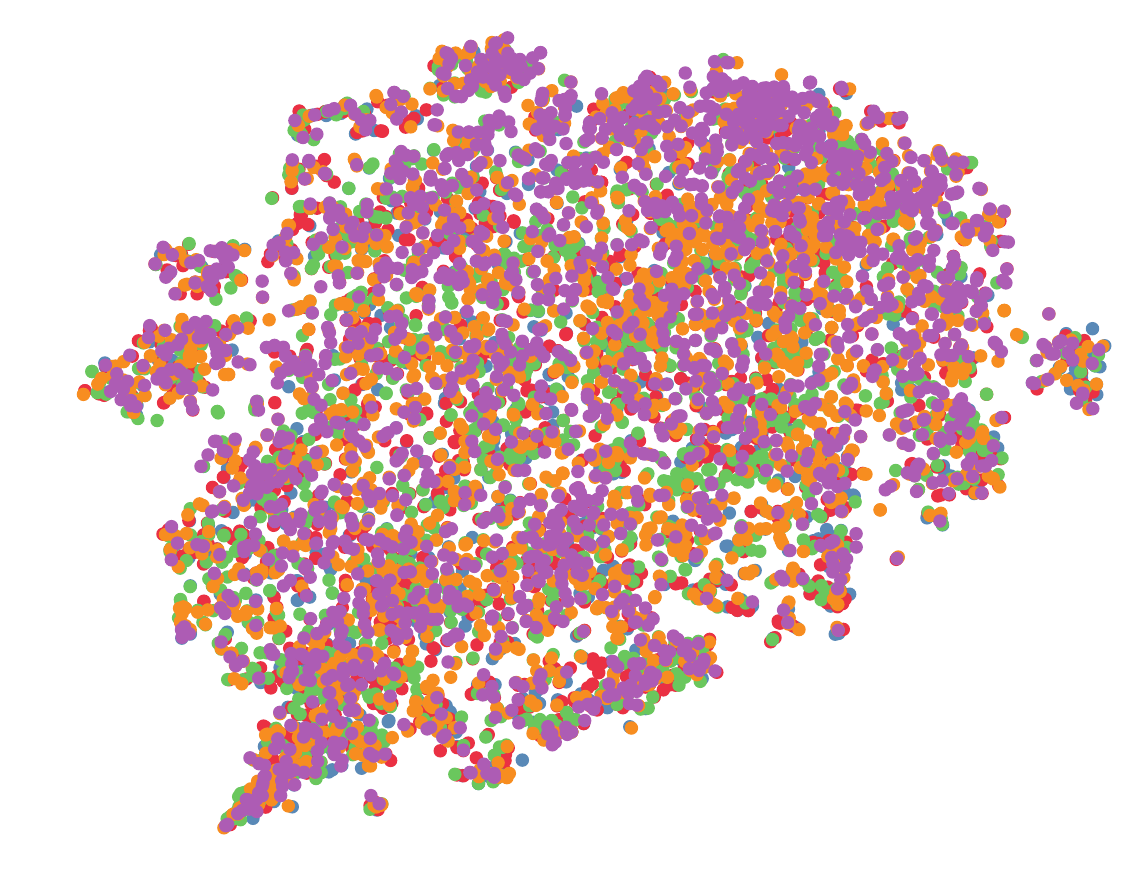}}
\hfill
\subfloat[Actor: HP Channel]{
\captionsetup{justification = centering}
\includegraphics[width=0.3\textwidth]{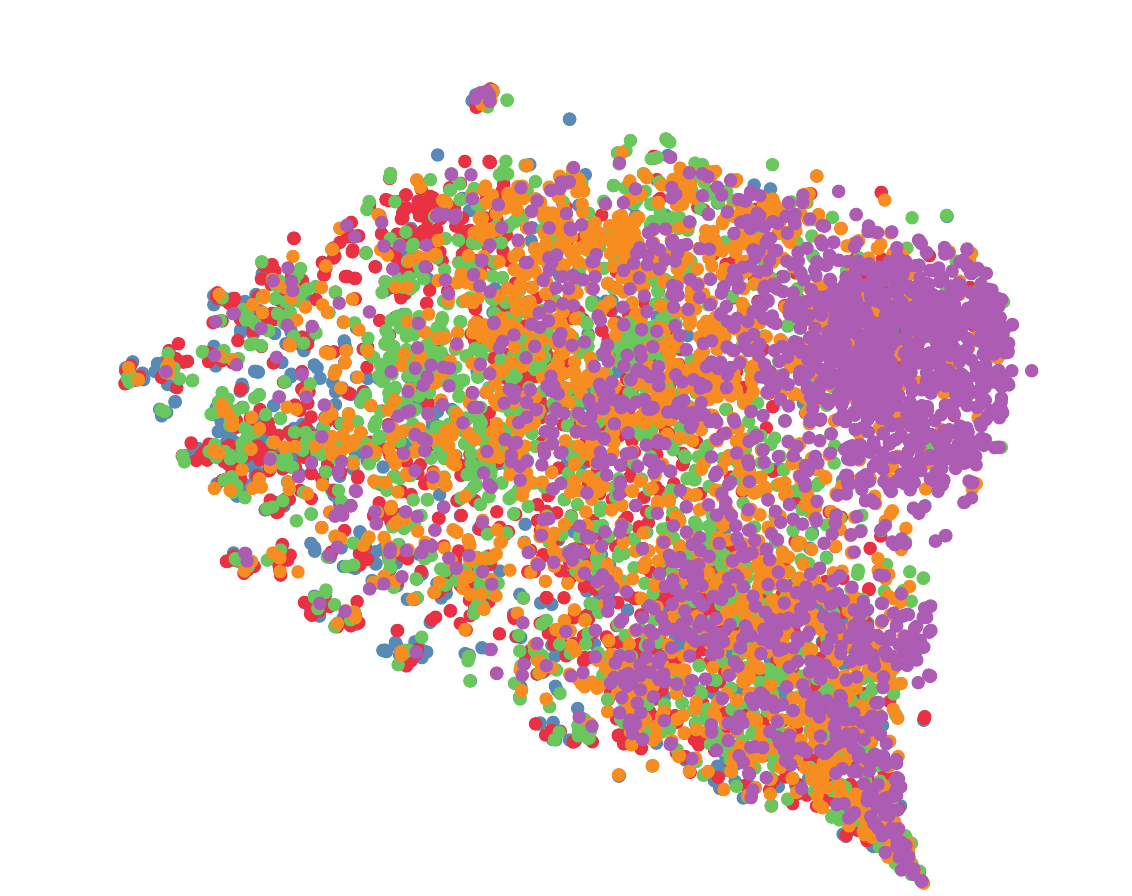}}
\hfill
\subfloat[Actor: Two Channels]{
\captionsetup{justification = centering}
\includegraphics[width=0.3\textwidth]{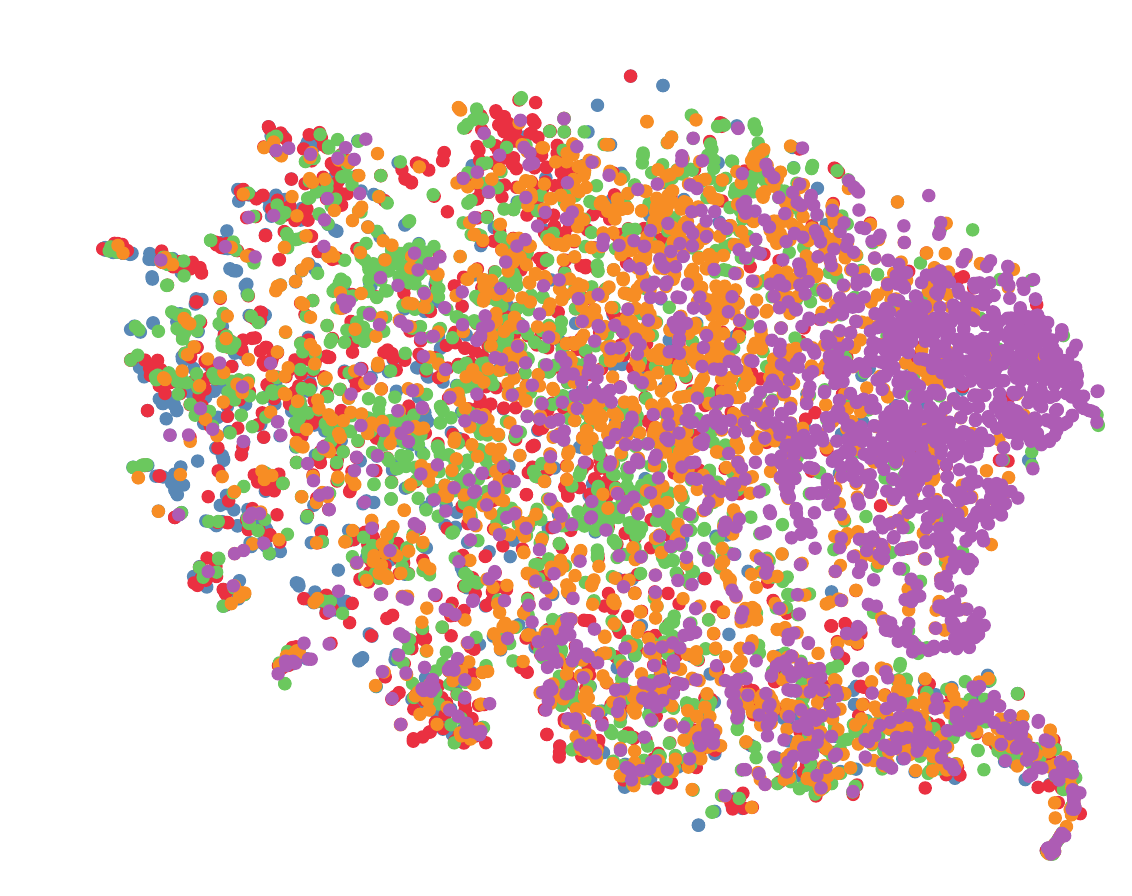}}
\\
\subfloat[Chameleon: LP Channel]{
\captionsetup{justification = centering}
\includegraphics[width=0.3\textwidth]{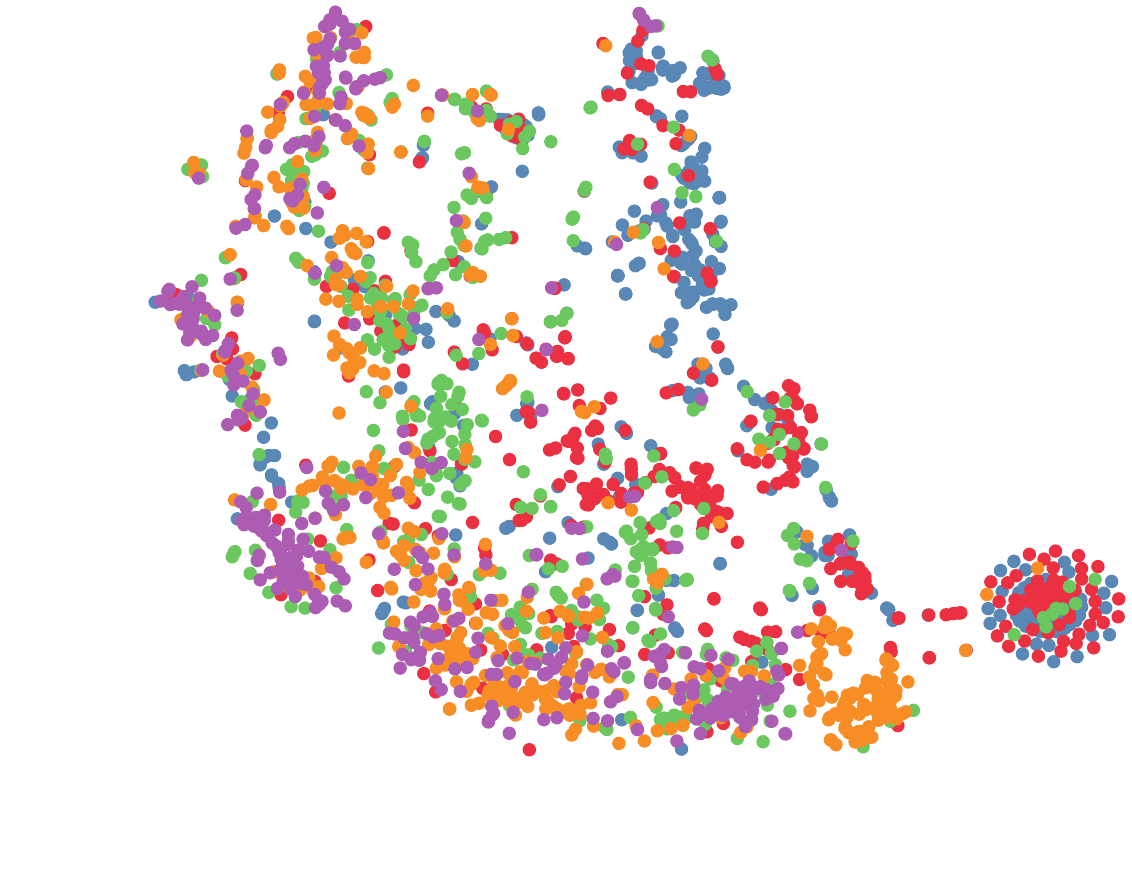}}
\hfill
\subfloat[Chameleon: HP Channel]{
\captionsetup{justification = centering}
\includegraphics[width=0.3\textwidth]{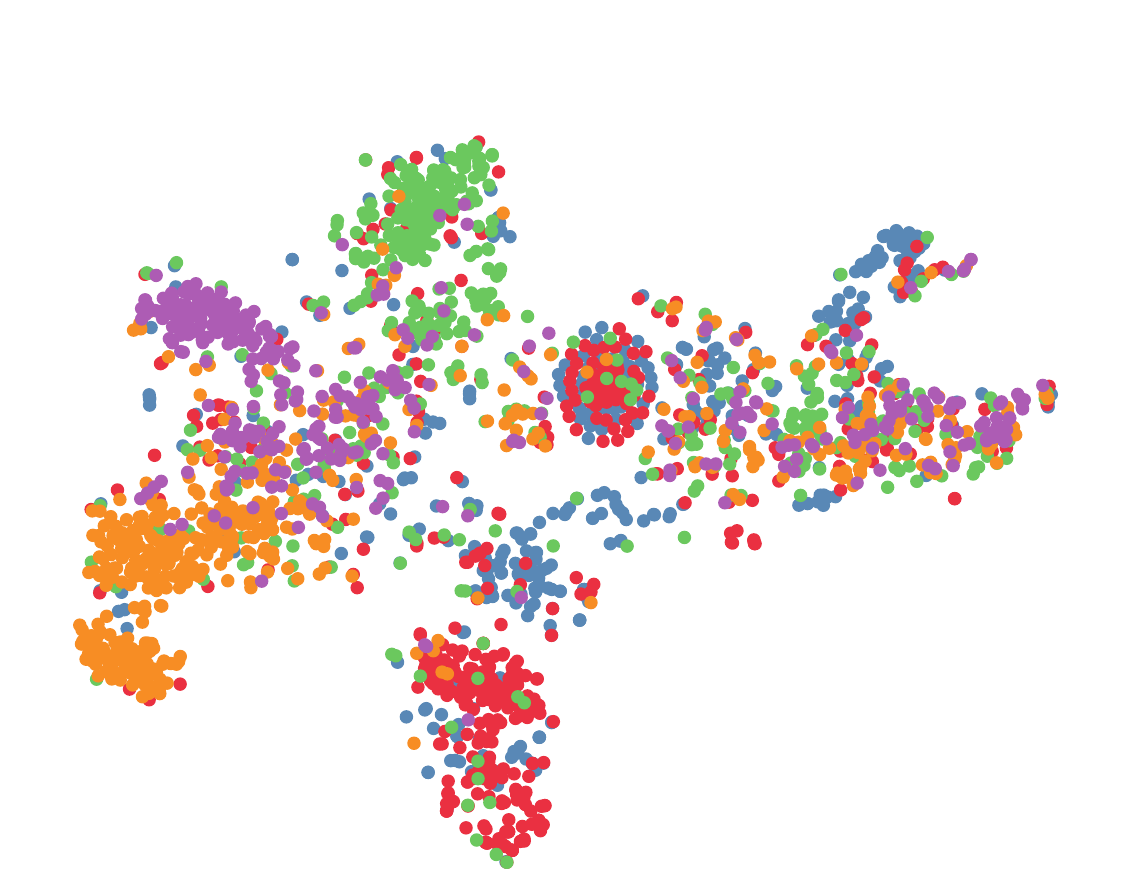}}
\hfill
\subfloat[Chameleon: Two Channels]{
\captionsetup{justification = centering}
\includegraphics[width=0.3\textwidth]{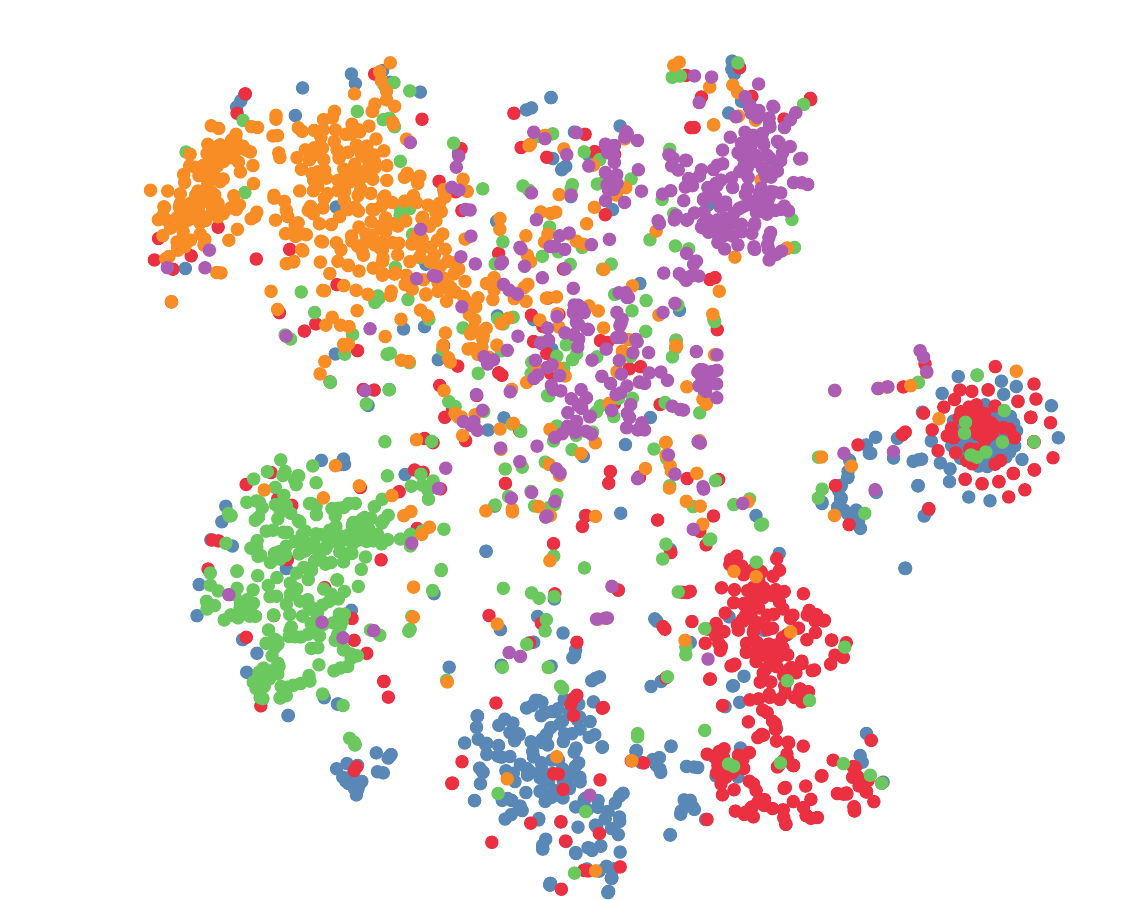}}
\\
\subfloat[Cornell: LP Channel]{
\captionsetup{justification = centering}
\includegraphics[width=0.3\textwidth]{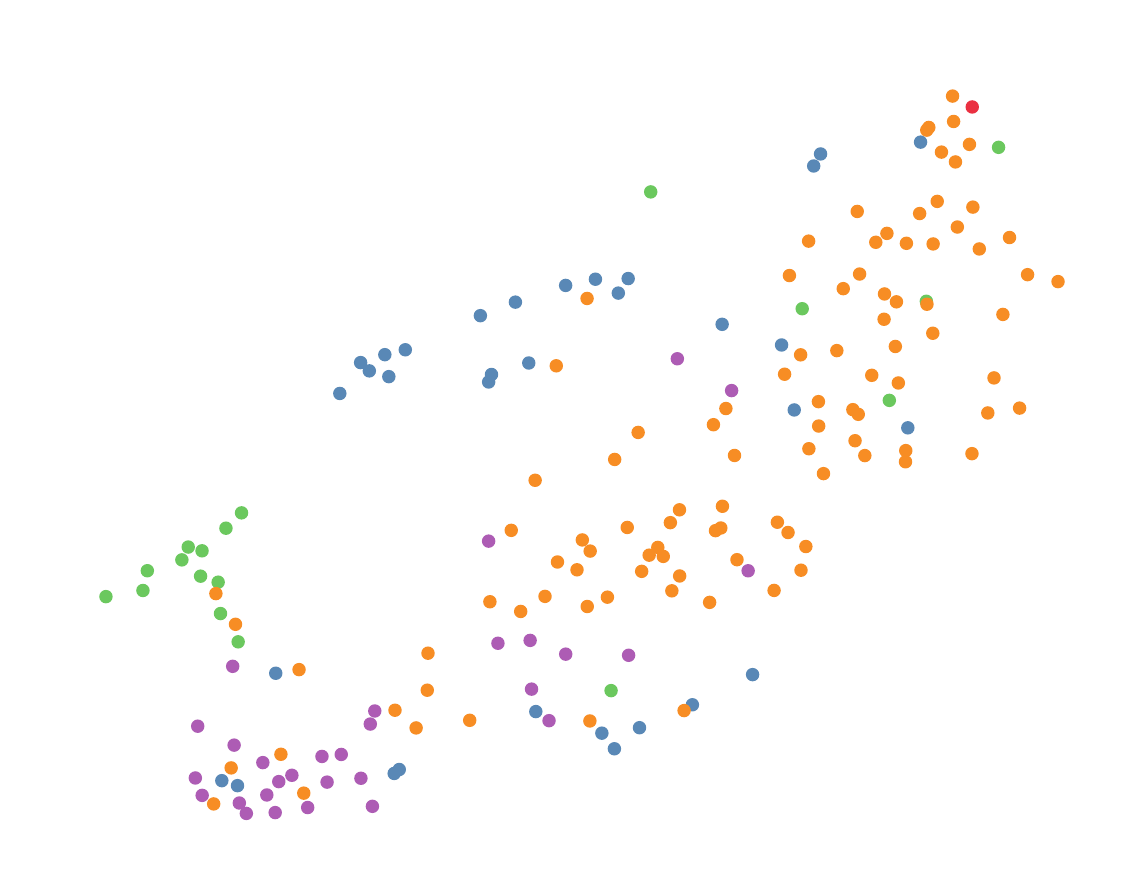}}
\hfill
\subfloat[Cornell: HP Channel]{
\captionsetup{justification = centering}
\includegraphics[width=0.3\textwidth]{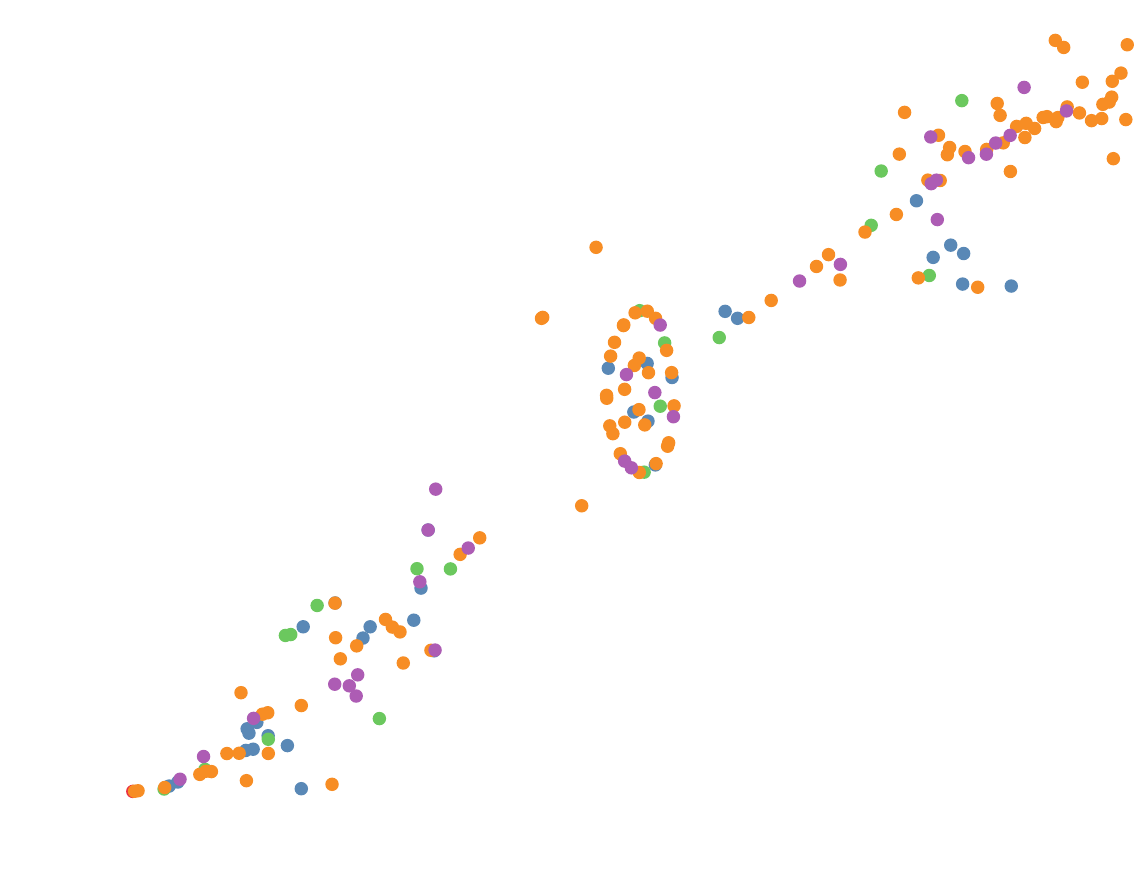}}
\hfill
\subfloat[Cornell: Two Channels]{
\captionsetup{justification = centering}
\includegraphics[width=0.3\textwidth]{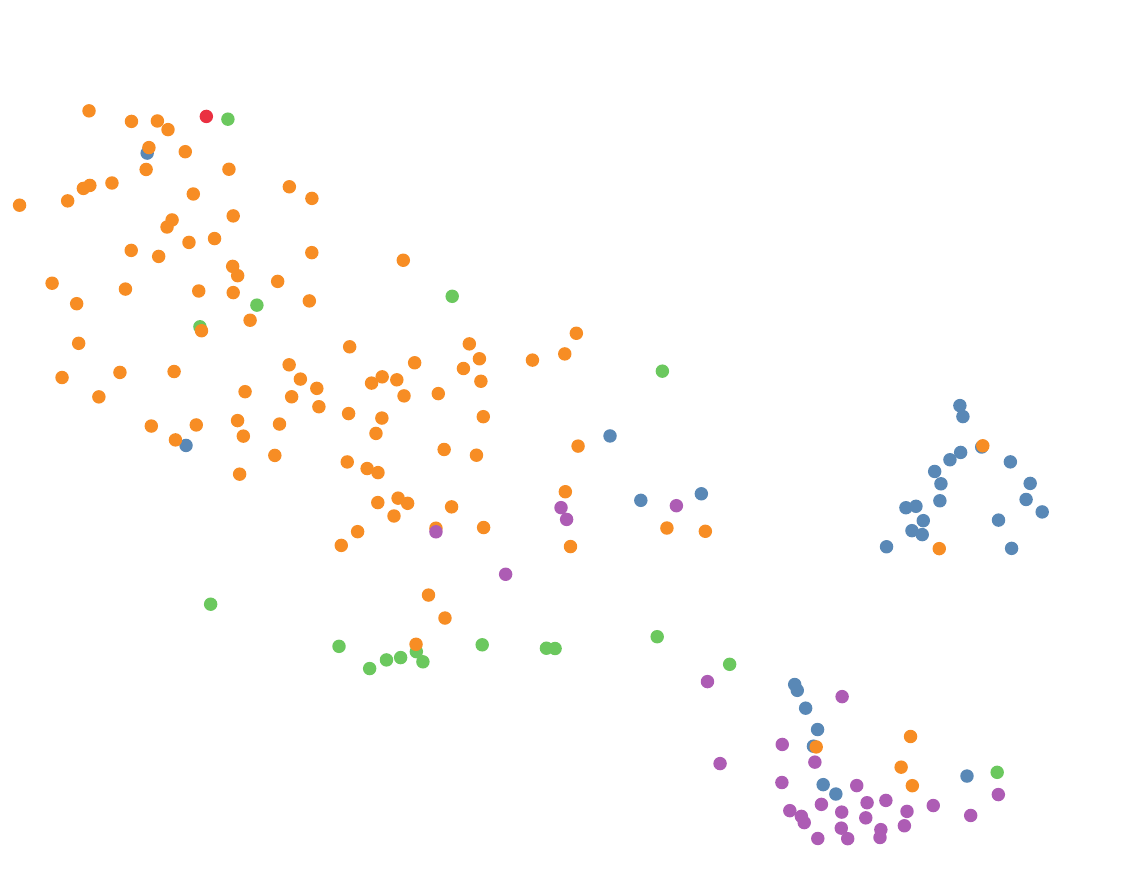}}
\\
\subfloat[Texas: LP Channel]{
\captionsetup{justification = centering}
\includegraphics[width=0.3\textwidth]{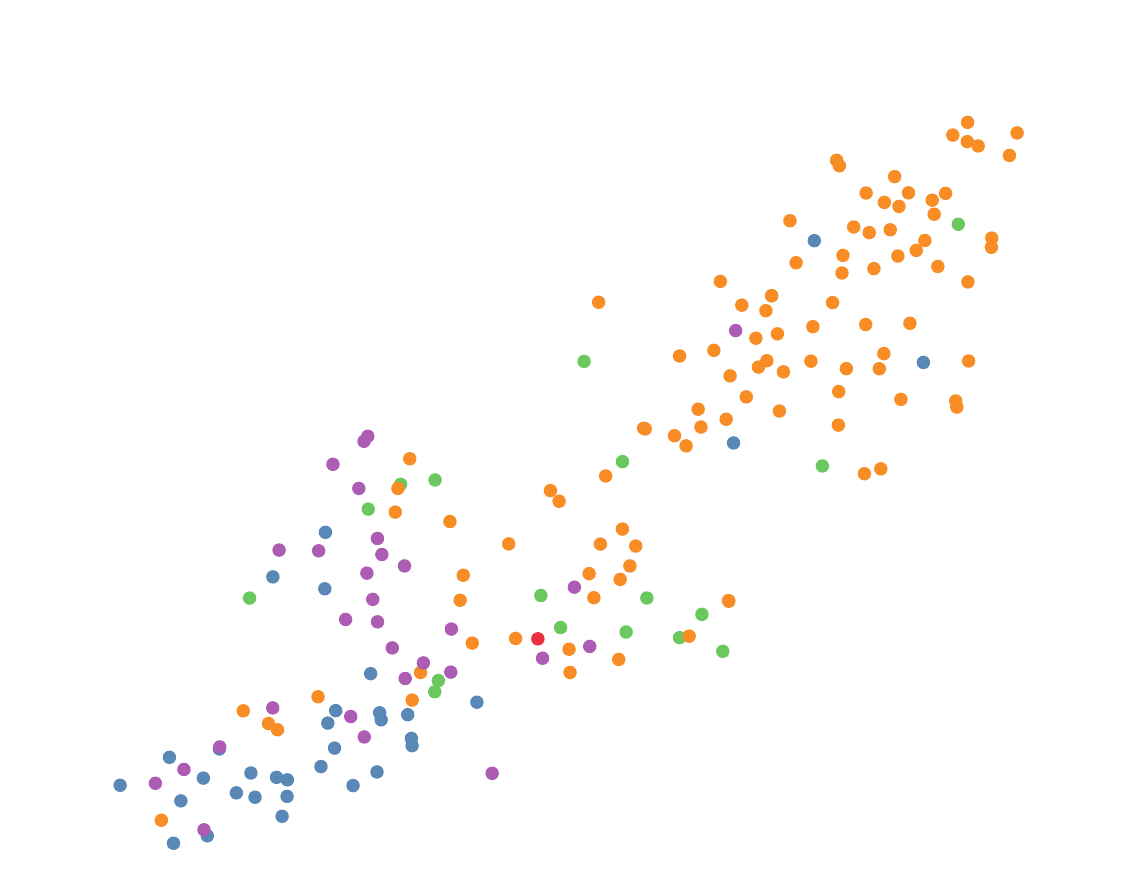}}
\hfill
\subfloat[Texas: HP Channel]{
\captionsetup{justification = centering}
\includegraphics[width=0.3\textwidth]{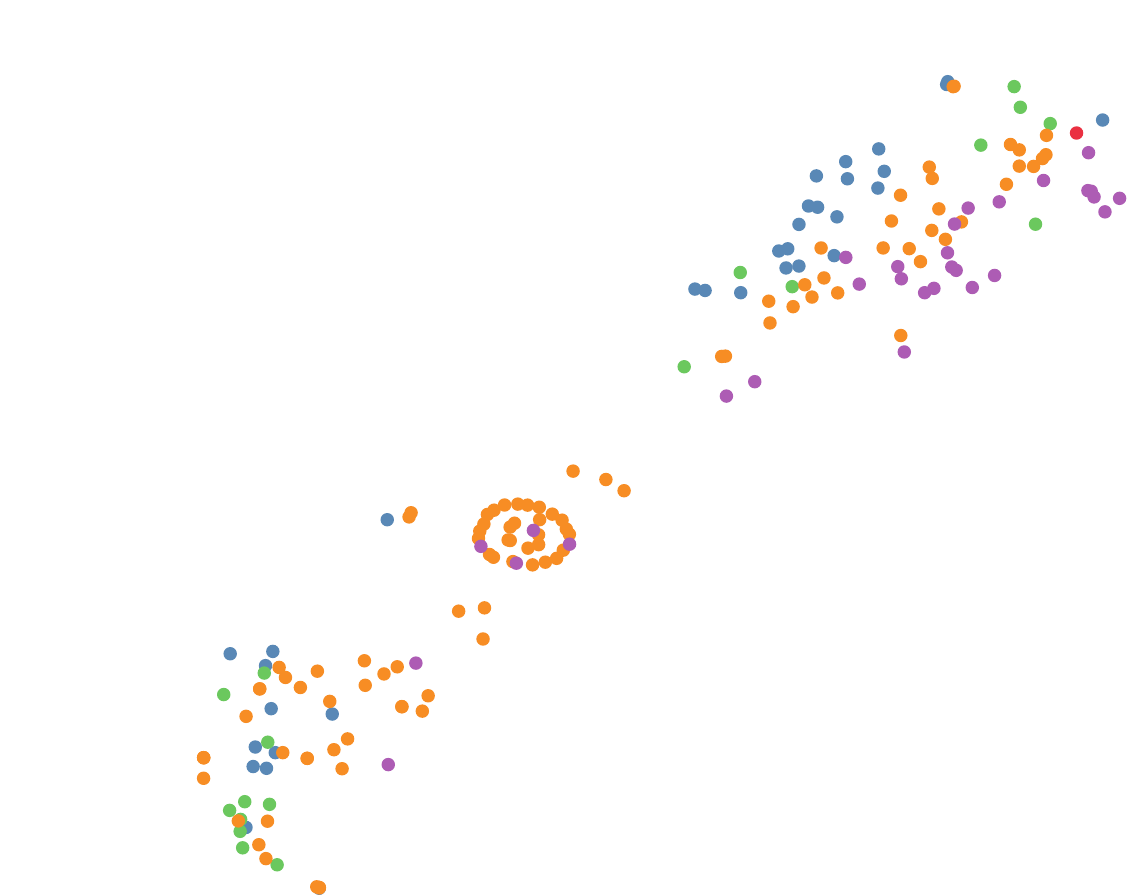}}
\hfill
\subfloat[Texas: Two Channels]{
\captionsetup{justification = centering}
\includegraphics[width=0.3\textwidth]{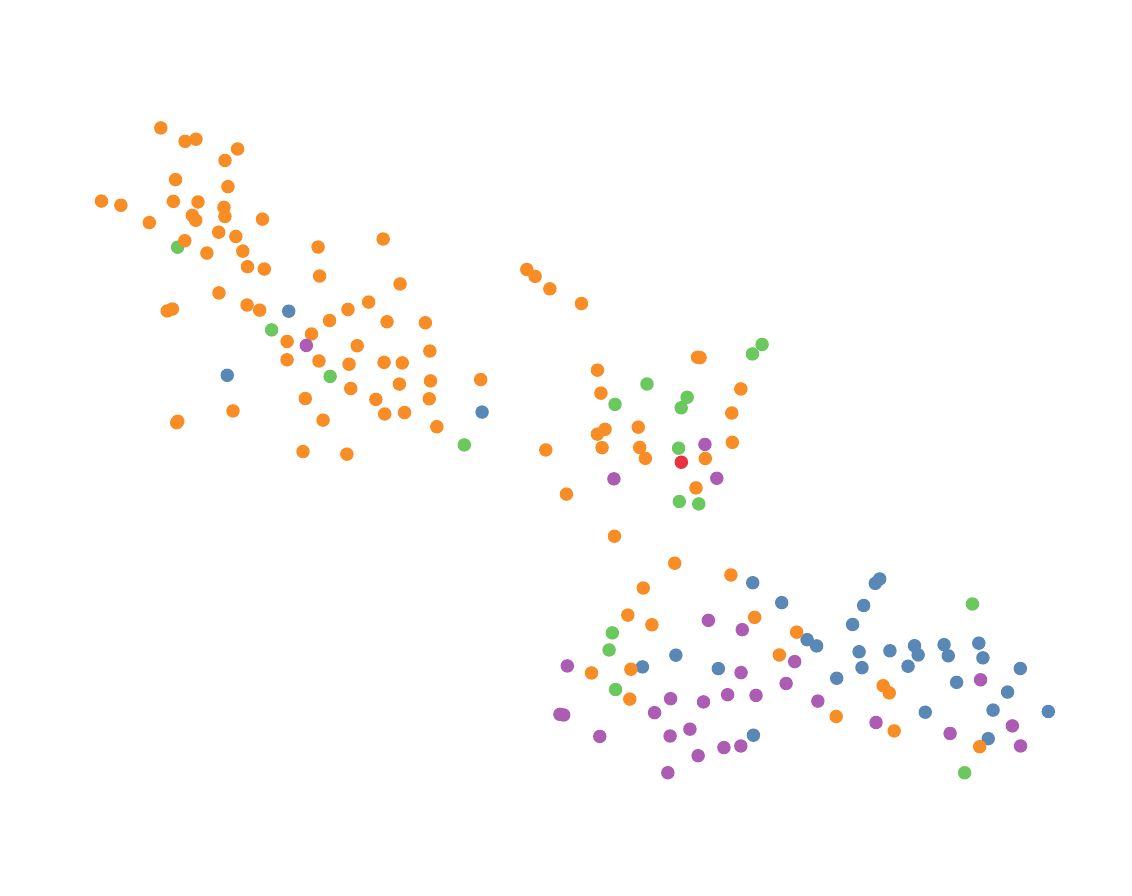}}
\\
\subfloat[Wisconsin: LP Channel]{
\captionsetup{justification = centering}
\includegraphics[width=0.3\textwidth]{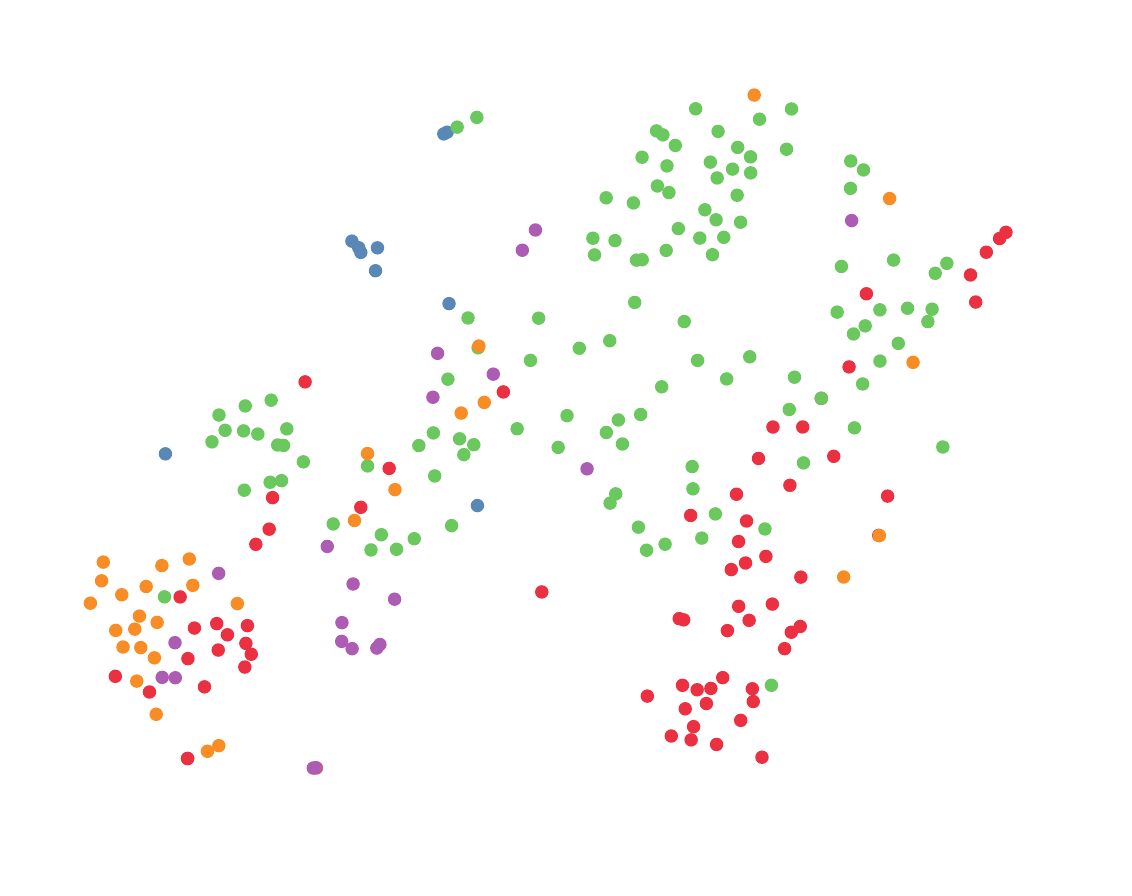}}
\hfill
\subfloat[Wisconsin: HP Channel]{
\captionsetup{justification = centering}
\includegraphics[width=0.3\textwidth]{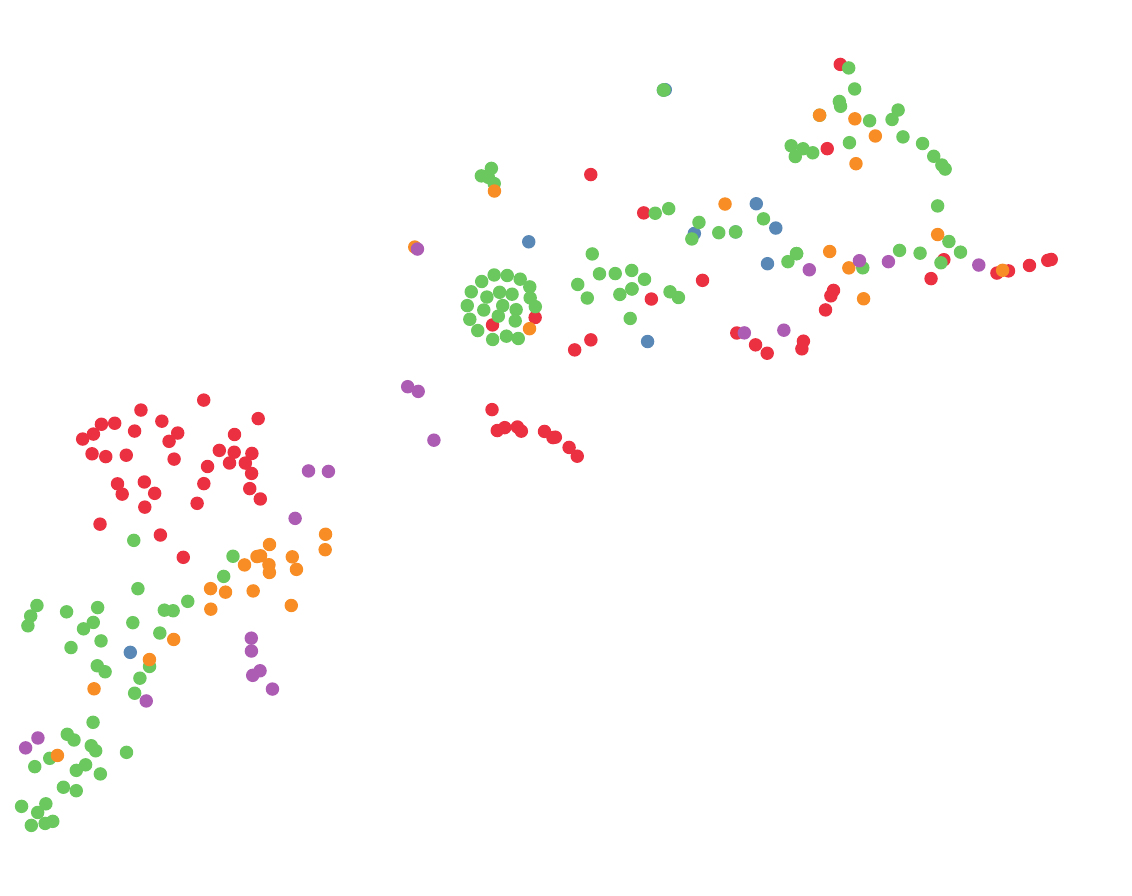}}
\hfill
\subfloat[Wisconsin: Two Channels]{
\captionsetup{justification = centering}
\includegraphics[width=0.3\textwidth]{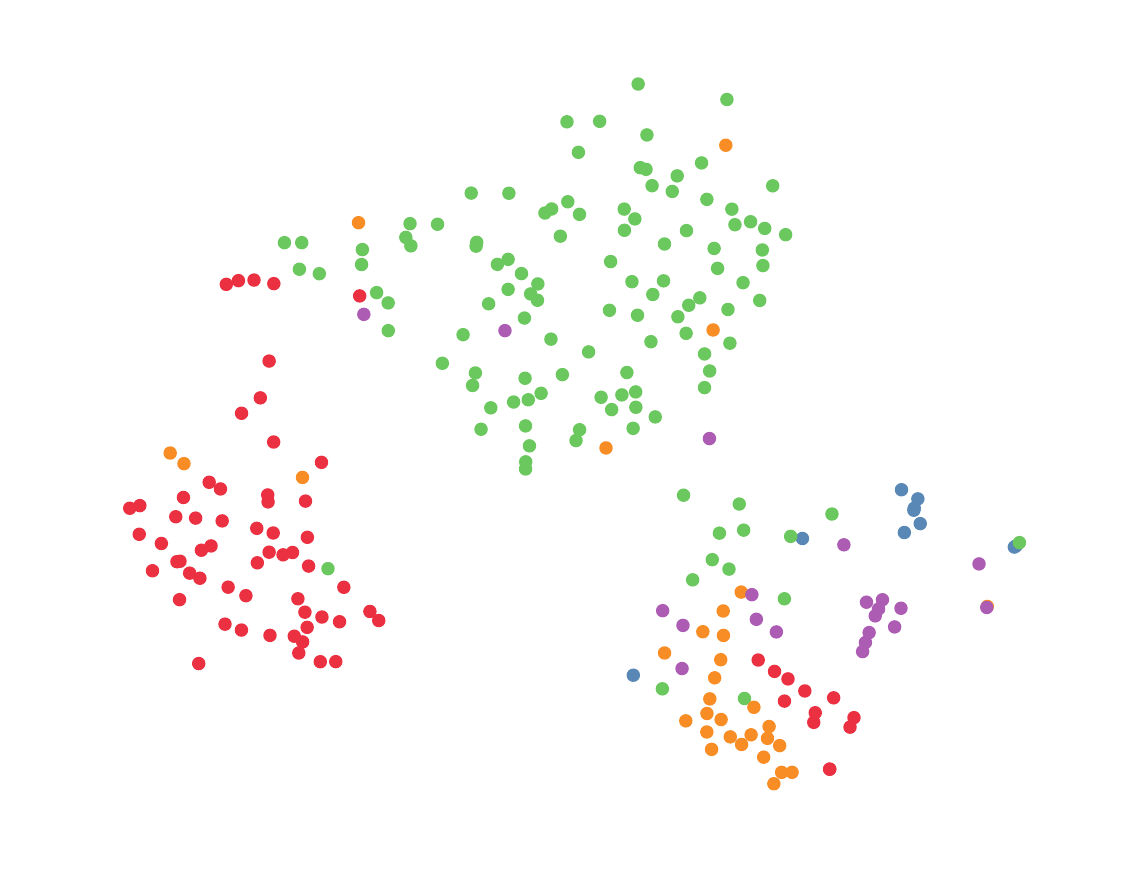}}
}
\caption{$t$-SNE Visualization of the Learned Node Embeddings for heterophilic datasets (other than Squirrel, which is presented in the main manuscript) under $3$ configurations.} %
\label{fig:tsne_heterophily}
\end{figure*}

\begin{figure*}[htbp]
\centering
{
\subfloat[CiteSeer: LP Channel]{
\captionsetup{justification = centering}
\includegraphics[width=0.3\textwidth]{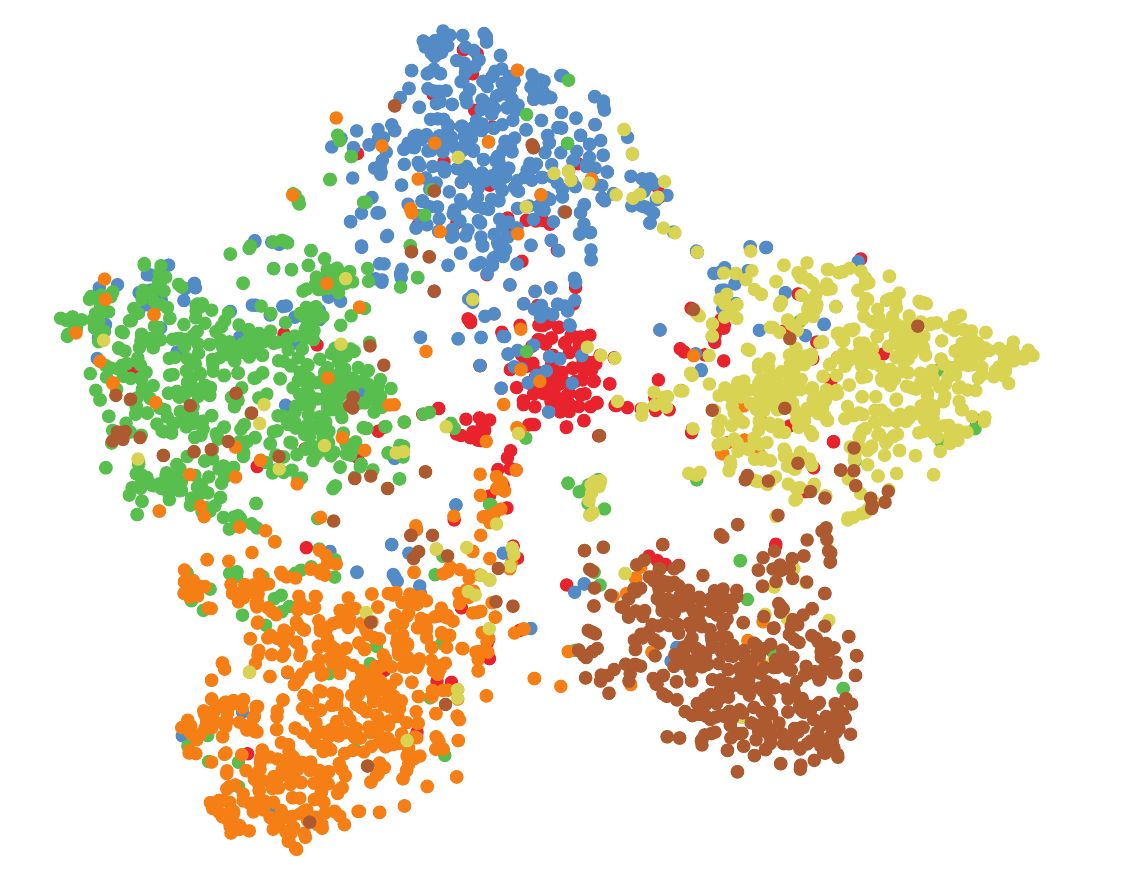}}
\hfill
\subfloat[CiteSeer: HP Channel]{
\captionsetup{justification = centering}
\includegraphics[width=0.3\textwidth]{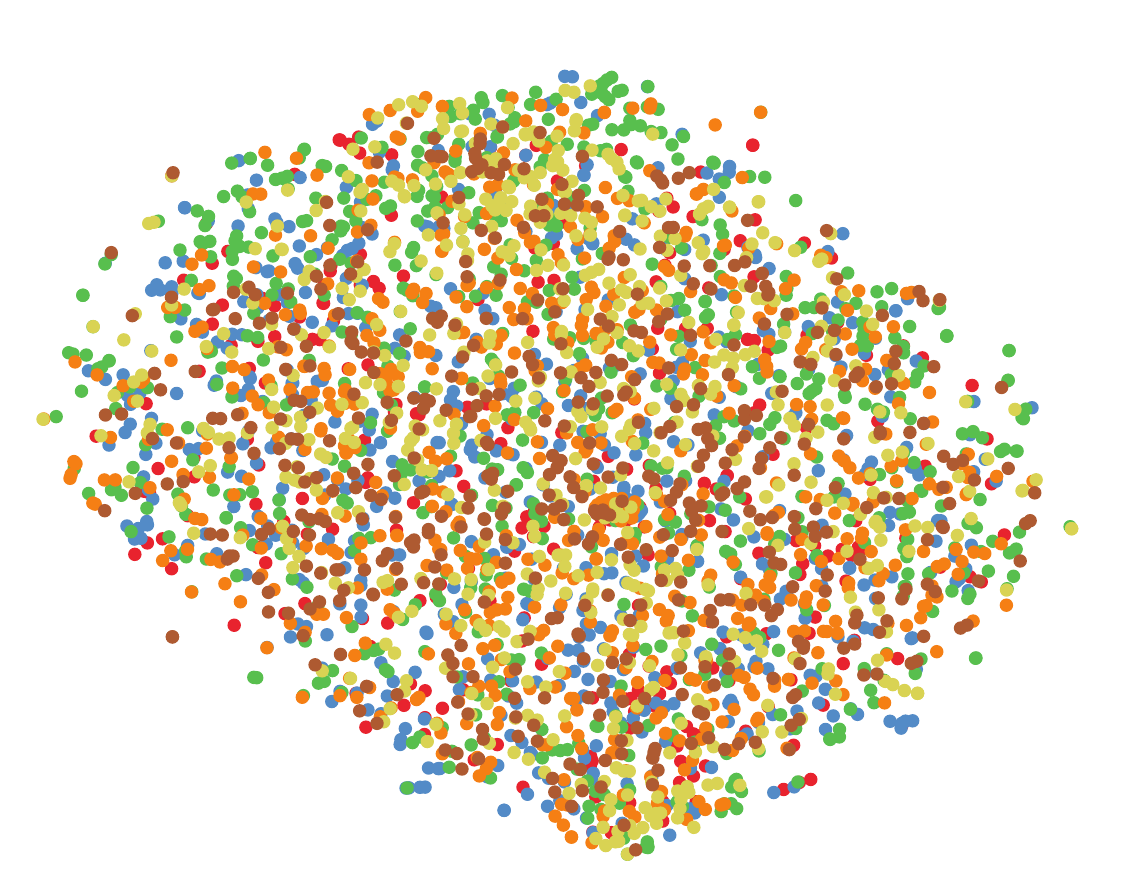}}
\hfill
\subfloat[CiteSeer: Two Channels]{
\captionsetup{justification = centering}
\includegraphics[width=0.3\textwidth]{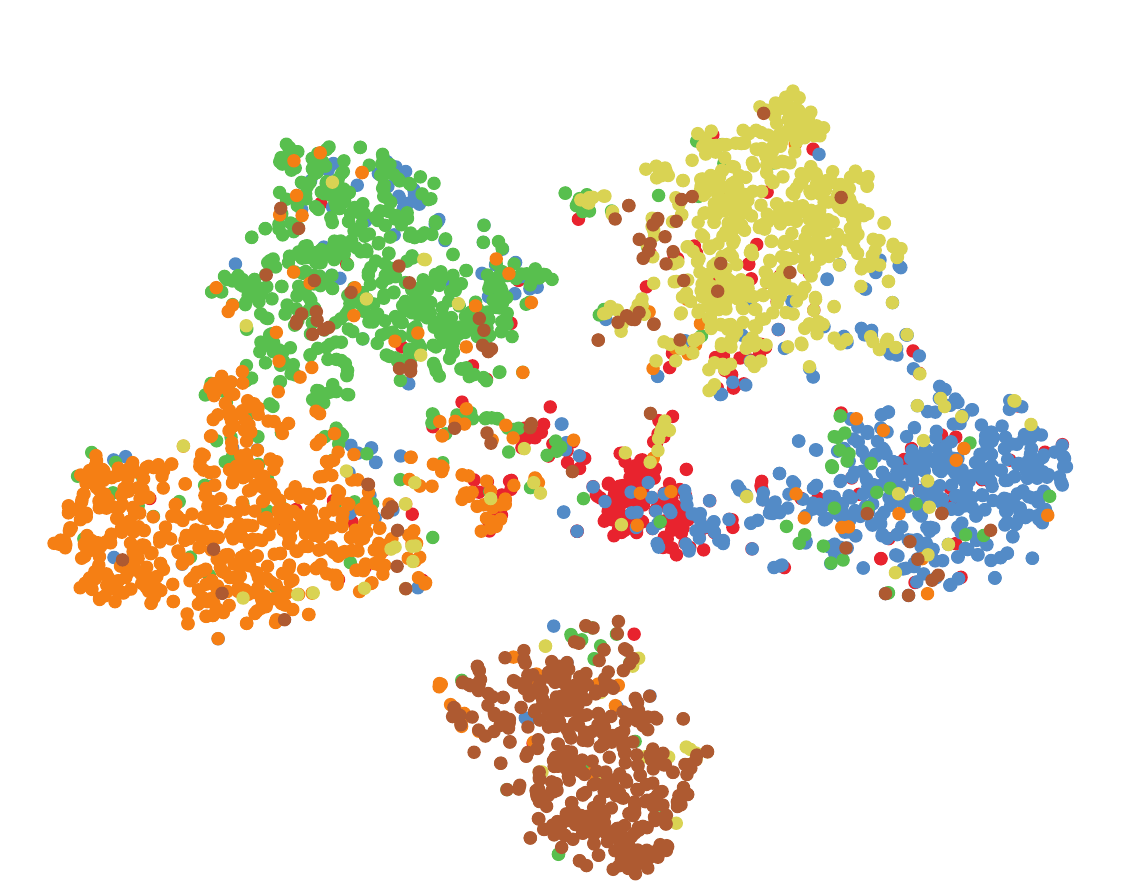}}
\\
\subfloat[Cora: LP Channel]{
\captionsetup{justification = centering}
\includegraphics[width=0.3\textwidth]{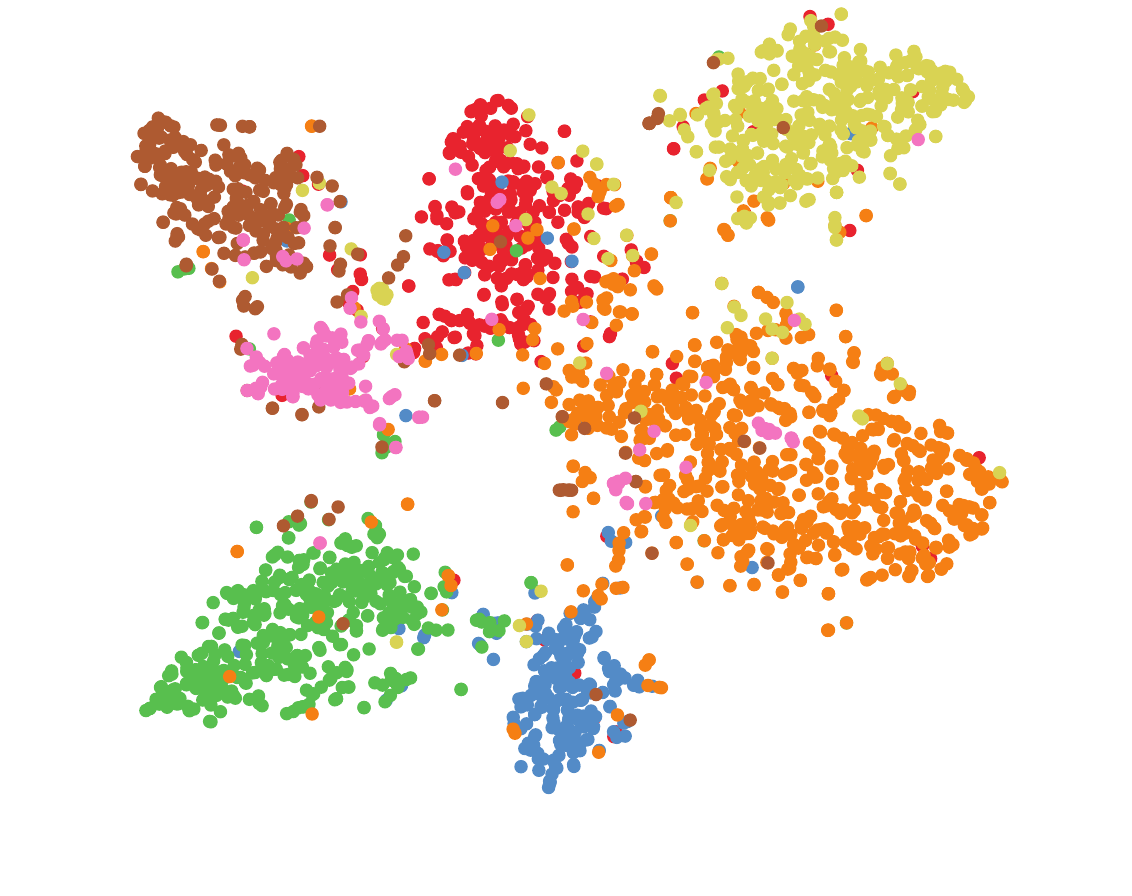}}
\hfill
\subfloat[Cora: HP Channel]{
\captionsetup{justification = centering}
\includegraphics[width=0.3\textwidth]{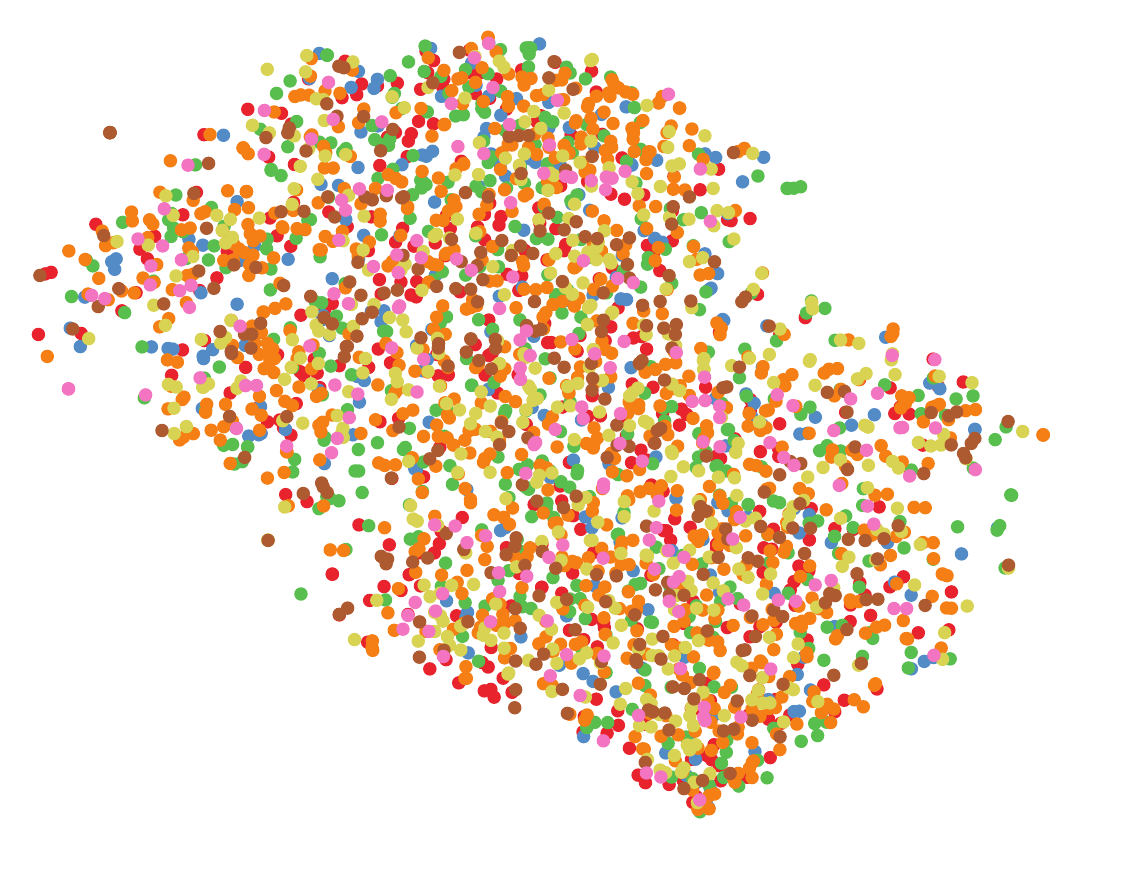}}
\hfill
\subfloat[Cora: Two Channels]{
\captionsetup{justification = centering}
\includegraphics[width=0.3\textwidth]{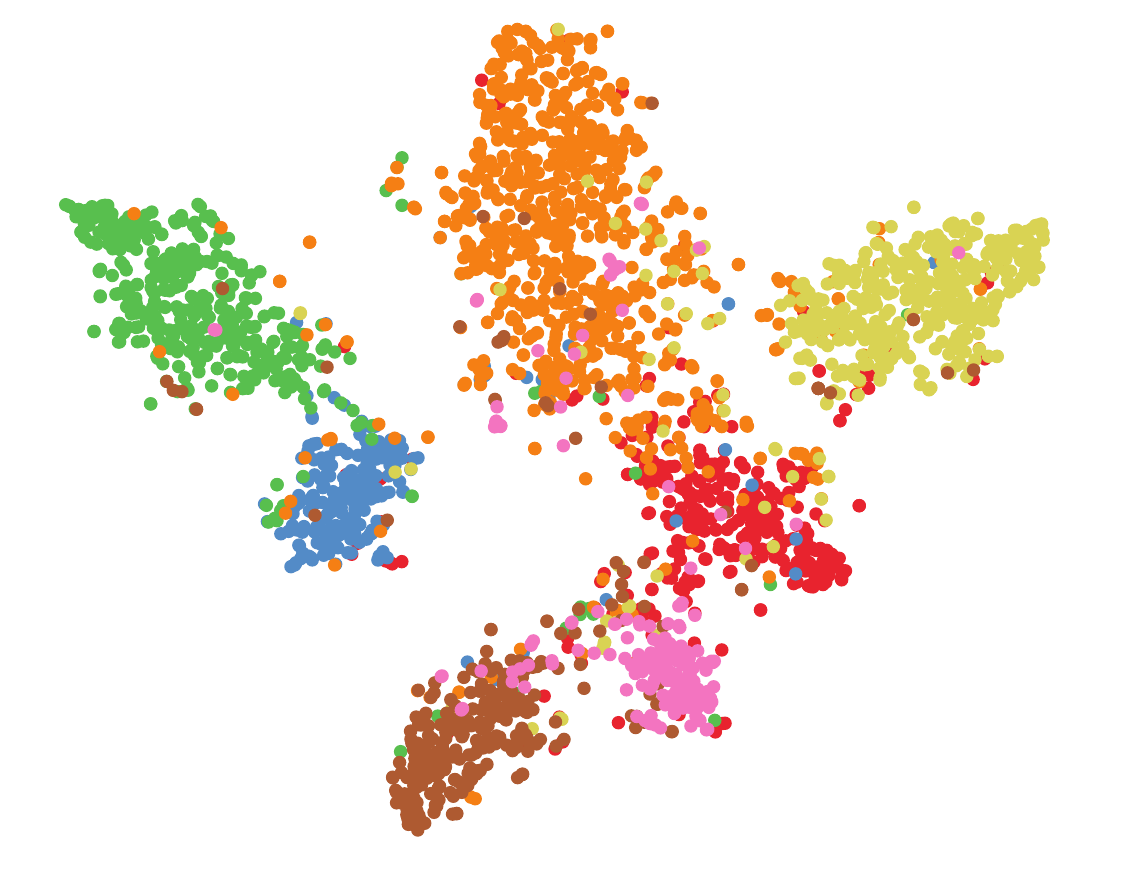}}
\\
\subfloat[PubMed: LP Channel]{
\captionsetup{justification = centering}
\includegraphics[width=0.3\textwidth]{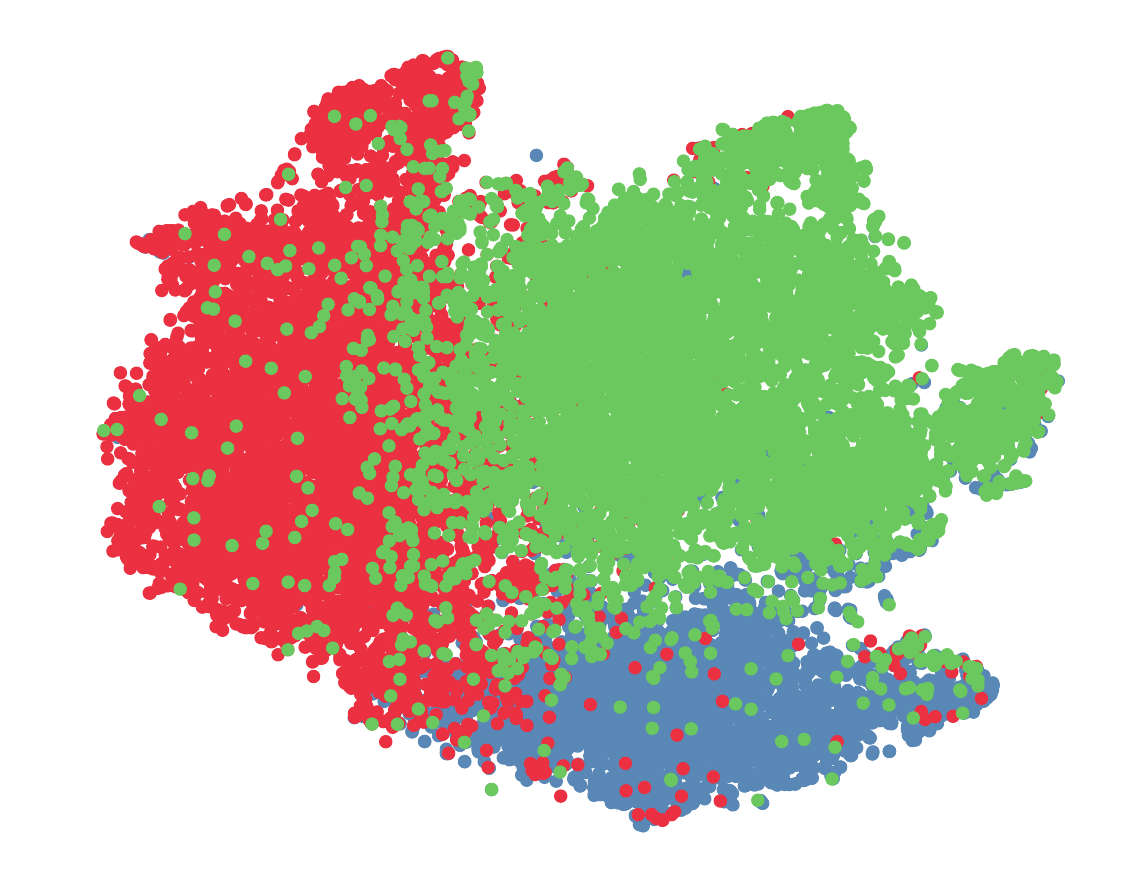}}
\hfill
\subfloat[PubMed: HP Channel]{
\captionsetup{justification = centering}
\includegraphics[width=0.3\textwidth]{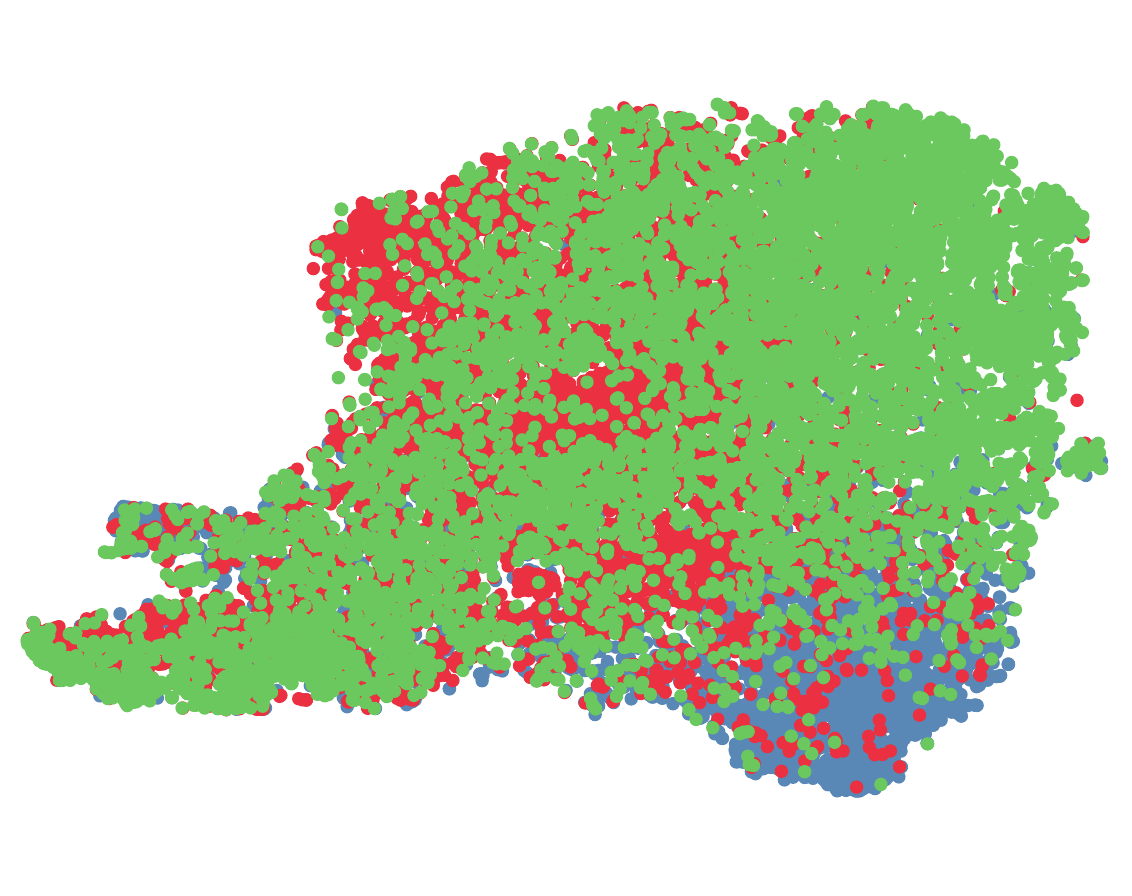}}
\hfill
\subfloat[PubMed: Two Channels]{
\captionsetup{justification = centering}
\includegraphics[width=0.3\textwidth]{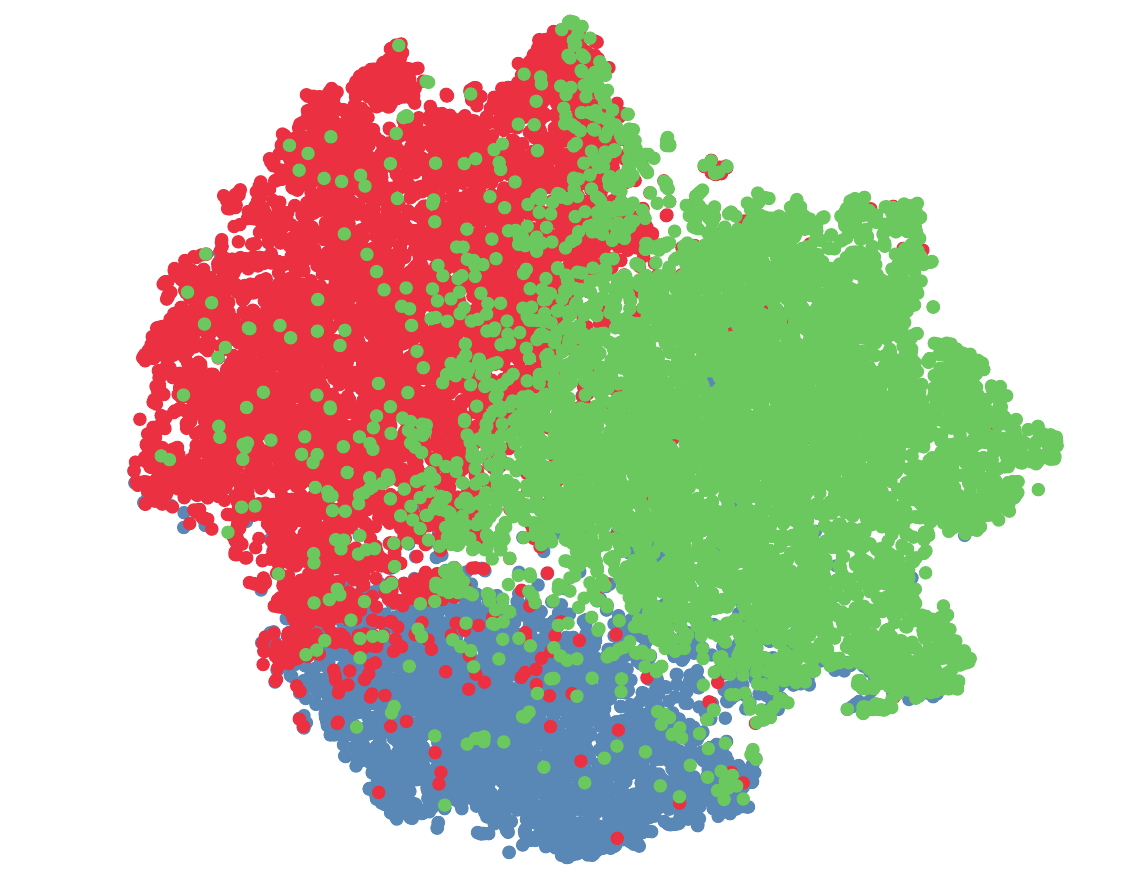}}
}
\caption{$t$-SNE Visualization of the Learned Node Embeddings for homophilic datasets under $3$ configurations.} %
\label{fig:tsne_homophily}
\end{figure*}
\clearpage

\section{Discussion of Filters}
\label{appendix:discussion_filters}
\subsection{Lazy Random Walk Matrix}
In graph signal processing, lazy random walk defined in the following equation is often used as a LP filter \cite{ekambaram2014graph},
\begin{equation}
    A_{\text{lrw}} = \frac{1}{2}(I+A_{\text{rw}}).
\end{equation}
It can be seen as adding $D_{ii}$ self-loops to the $i$-th node of $A$ and normalize it to be a random walk matrix and $0 \leq \lambda_i(A_{\text{lrw}}) \leq 1$. Such spectral property makes it a standard band-pass filter, which can avoid some theoretical confusion due to negative eigenvalues. Unlike $\hat{A}_\text{rw}$, $A_{\text{lrw}}$ maintains certain topology properties of $A_\text{rw}$, \eg{} stationary distribution and eigenvectors. In practice, these properties are supposed to be changed, unless one has strong prior knowledge that GNNs will benefit from the renormalized one. 

Furthermore, it is found that adding self-loops can shrink the magnitude of the dominant eigenvalue so that the influence of long-distance nodes will be reduced, which makes the filtered signal more dependent on local information \cite{hamilton2020graph}. There is growing empirical evidence showing that adding self-loops will lead to effective graph convolutions on some applications \cite{klicpera2018predict, wu2019simplifying}. Compared to $\hat{A}_\text{rw}$, $A_{\text{lrw}}$ works better at reducing the magnitude of dominant eigenvalue under certain conditions. We show it in the following theorem.
\begin{theorem} 1
We denote the generalized lazy random walk matrix and  the generalized renormalized adjacency matrix 
respectively by 
$$
A_{\text{lrw}}^\gamma = \frac{1}{1+\gamma}(\gamma I + A_{\text{rw}}),\ \ 
\hat{A}_{\text{rw}}^\gamma = \tilde{D}_\gamma^{-1} \tilde{A}_\gamma,
$$
where $\tilde{A}_\gamma = \gamma I + A, \tilde{D}_\gamma = \gamma I + D$. Suppose $\mathcal{G}$ has no isolated node, \ie{} $D_{ii} > 0$ for all $i$, and for positive $\gamma$, we have
\begin{equation}
    \frac{\lambda_2(A_{\text{lrw}}^\gamma)}{\lambda_1(A_{\text{lrw}}^\gamma)} \geq \frac{\lambda_2(\hat{A}_{\text{rw}}^\gamma)}{\lambda_1(\hat{A}_{\text{rw}}^\gamma)}, 
\end{equation}
where $\lambda_1(\cdot)$ and $\lambda_2(\cdot)$ are the largest and second largest eigenvalues a matrix.
\end{theorem}
\begin{proof}
    Detailed proof can be found in Appendix \ref{appendix:proof}.
\end{proof}

The HP filer derived from $A_{\text{lrw}}$ is $(I-A_{\text{rw}})/2$ and they can be used as a set of filterbank in FB-GNN framework. From table \ref{tab:lazy_rw} we can see that, FB-GNNs with lazy random walk matrix can boost the performance of baseline GNNs more significantly than symmetric renormalized affinity matrix on heterophilic datasets \textit{Cornell, Wisconsin, Texas} and \textit{Film}, where baseline GNNs underperform MLP. And on homophilic datasets, where baseline GNNs outperform MLP, FB-GNNs with symmetric renormalized affinity matrix perform better.
\clearpage
\begin{table}[htbp]
  \centering
  \small
  \caption{Comparison of FB-GNNs with lazy random walk and symmetric renormalized affinity matrix}
    \begin{tabular}{c|ccccccccc}
    \toprule
    \toprule
    Models\textbackslash{}Datasets & Cornell & Wisconsin & Texas & Film  & Chameleon & Squirrel & Cora  & Citeseer  & Pubmed  \\
    \midrule
    MLP   & 85.14 & 87.25 & 84.59 & 36.08 & 46.21 & 29.39 & 74.81 & 73.45 & 87.86 \\
    \midrule
    GCN   & 60.81 & 63.73 & 61.62 & 30.98 & 61.34 & 41.86 & 87.32 & 76.70 & 88.24 \\
    FB-GCM-sym & 83.78 & 87.45 & 84.59 & 35.47 & \cellcolor[rgb]{ .647,  .647,  .647}\textbf{65.66} & \cellcolor[rgb]{ .647,  .647,  .647}\textbf{50.56} & 87.34 & 76.42 & \cellcolor[rgb]{ .647,  .647,  .647}\textbf{89.74} \\
    FB-GCN-lazy & 85.14 & 87.25 & 86.49 & 36.11 & 61.07 & 46.28 & 87.46 & 76.24 & 89.57 \\
    \midrule
    GAT   & 59.19 & 60.78 & 59.73 & 29.71 & 61.95 & 43.88 & \cellcolor[rgb]{ .647,  .647,  .647}\textbf{88.07} & 76.42 & 87.81 \\
    FB-GAT-sym & 87.30 & 85.88 & 82.70 & 36.00 & 63.75 & 44.07 & 87.34 & 76.45 & 89.03 \\
    FB-GAT-lazy & \cellcolor[rgb]{ .647,  .647,  .647}\textbf{88.92} & \cellcolor[rgb]{ .647,  .647,  .647}\textbf{89.22} & \cellcolor[rgb]{ .647,  .647,  .647}\textbf{88.65} & \cellcolor[rgb]{ .647,  .647,  .647}\textbf{36.83} & 62.80 & 43.56 & 87.28 & 76.92 & 89.43 \\
    \midrule
    GraphSage & 82.97 & 87.84 & 82.43 & 35.28 & 47.32 & 30.16 & 85.98 & 77.07 & 88.59 \\
    FB-GraphSage-sym & 86.44 & 88.43 & 86.22 & 35.88 & 48.11 & 33.24 & 86.62 & \cellcolor[rgb]{ .647,  .647,  .647}\textbf{78.01} & 89.05 \\
    FB-GraphSage-lazy & 86.49 & 87.45 & 87.30 & 35.00 & 48.57 & 30.44 & 86.35 & 76.48 & 88.58 \\
    \midrule
    diff(FB-sym, baseline) & 18.18 & 16.47 & 16.58 & 3.79  & 2.30  & 3.99  & -0.02 & 0.23  & 1.06 \\
    diff(FB-lazy, baseline) & 19.19 & 17.19 & 19.55 & 3.99  & 0.61  & 1.46  & -0.09 & -0.18 & 0.98 \\
    \bottomrule
    \bottomrule
    \end{tabular}%
  \label{tab:lazy_rw}
\end{table}%

\begin{table}[htbp]
  \centering
  \small
  \caption{More comparisons of FB-GNN with lazy random walk and symmetric renormalized matrix}
    \begin{tabular}{c|ccccc}
    \toprule
    \toprule
    Models\textbackslash{}Datasets & ﻿CitationFull\_dblp &  Coauthor\_CS &  Coauthor\_Physics &  Amazon\_Computers & Amazon\_Photo  \\
    \midrule
    MLP   & 77.39 & 93.72 & 95.77 & 83.89 & 90.87 \\
    \midrule
    GCN   & 85.87 & 93.91 & 96.84 & 87.03 & 93.61 \\
    FB-GCM-sym & \cellcolor[rgb]{ .647,  .647,  .647}\textbf{85.90} & \cellcolor[rgb]{ .647,  .647,  .647}\textbf{95.33} & 97.03 & \cellcolor[rgb]{ .647,  .647,  .647}\textbf{91.54} & \cellcolor[rgb]{ .647,  .647,  .647}\textbf{95.57} \\
    FB-GCN-lazy & 85.51 & 95.31 & \cellcolor[rgb]{ .647,  .647,  .647}\textbf{97.07} & 91.32 & 95.53 \\
    \midrule
    GAT   & 85.89 & 93.41 & 96.32 & 89.74 & 94.12 \\
    FB-GAT-sym & 84.94 & 94.13 & OOM   & 89.57 & 94.58 \\
    FB-GAT-lazy & 85.35 & 95.30 & OOM   & 91.10 & 94.88 \\
    \midrule
    GraphSage & 81.19 & 94.38 & OOM   & 83.70 & NA \\
    FB-GraphSage-sym & 85.66 & 94.97 & OOM   & 88.01 & 91.40 \\
    FB-GraphSage-lazy & 85.02 & 95.00 & OOM   & 87.12 & 91.02 \\
    \midrule
    diff(FB-sym, baseline) & 1.18  & 0.91  & 0.45  & 2.88  & -0.02 \\
    diff(FB-lazy, baseline) & 0.98  & 1.30  & 0.49  & 3.02  & -0.06 \\
    \bottomrule
    \bottomrule
    \end{tabular}%
  \label{tab:lazy_rw_other_results}%
\end{table}%

\subsection{Proof of Eigengap}
\label{appendix:proof}
\begin{proof}
Denote the symmetric normalized lazy random walk matrix
and the generalized symmetric renormalized adjacency matrix respectively by
$$
A_{\text{slrw}}^\gamma = \frac{1}{1+\gamma}(\gamma I + A_{\text{sym}}),\ \ 
\hat{A}_\text{sym}^\gamma = \tilde{D}_\gamma^{-1/2} \tilde{A}_\gamma \tilde{D}_\gamma^{-1/2}
$$

It is easy to verify that 
$$
\lambda(A_{\text{slrw}}^\gamma) = \lambda(A_{\text{lrw}}^\gamma), \ \ 
\lambda( \hat{A}_\text{sym}) = \lambda( \hat{A}_\text{rw}^\gamma)
$$
where $\lambda(\cdot)$ denotes the spectrum of a matrix.
Since $\lambda_1(A_\text{lrw}^\gamma) = \lambda_1(\hat{A}_\text{rw}^\gamma) = 1$, 
to prove the theorem, it is necessary and sufficient to prove
\begin{equation}
    \lambda_2(A_{\text{slrw}}^\gamma) \geq \lambda_2(\hat{A}_{\text{sym}}^\gamma) 
\end{equation}
It is easy to show that $D^{1/2} \bm{1}$ is an eigenvector of $A_{\text{slrw}}^\gamma$
corresponding to $\lambda_1(A_{\text{slrw}}^\gamma)$
and $\tilde{D}_\gamma^{1/2} \bm{1}$ is an eigenvector $\hat{A}_\text{sym}^\gamma$
corresponding to $\lambda_1(\hat{A}_\text{sym}^\gamma)$. 

By the Rayleigh quotient theorem (\cite{HorJ13}),
\begin{align} \label{eq:aslrw}
\lambda_2(A_{\text{slrw}}^\gamma)
=   \max_{\bm{x}\perp D^{1/2}\bm{1}} \frac{\bm{x}^T D^{-1/2}(\gamma D + A) D^{-1/2} \bm{x}}{(1+\gamma) \bm{x}^T \bm{x}}  
\end{align}
By the Rayleigh quotient theorem and the Courant-Fischer min-max theorem (\cite{HorJ13}), 
\begin{align}
 \lambda_2(\hat{A}_{\text{sym}}^\gamma)
= & \max_{\bm{x} \perp  \tilde{D}_\gamma^{1/2} \bm{1}} \frac{\bm{x}^T \tilde{D}_\gamma^{-1/2}(\gamma I + A) \tilde{D}_\gamma^{-1/2} \bm{x}}{\bm{x}^T \bm{x}} \nonumber \\ 
= & \min_{\{S: \dim{S}=n-1\}} \max_{\{\x:0\neq \x\in S\}} 
\frac{\bm{x}^T \tilde{D}_\gamma^{-1/2}(\gamma I + A) \tilde{D}_\gamma^{-1/2} \bm{x}}{\bm{x}^T \bm{x}} 
\nonumber\\
\leq & \max_{\bm{x}\perp D^{1/2}\bm{1}} 
\frac{\bm{x}^T \tilde{D}_\gamma^{-1/2}(\gamma I + A) \tilde{D}_\gamma^{-1/2} \bm{x}}{\bm{x}^T \bm{x}}
\label{eq:asym}
\end{align}
Then from \eqref{eq:aslrw} and \eqref{eq:asym} we obtain
\begin{align*}
&\lambda_2(A_{\text{slrw}}^\gamma) -  \lambda_2(\hat{A}_{\text{sym}}^\gamma)\\ 
& \geq \max_{\bm{x}\perp D^{1/2}\bm{1}} \frac{\bm{x}^T D^{-1/2}(\gamma D + A) D^{-1/2} \bm{x}}{(1+\gamma) \bm{x}^T \bm{x}} \\ 
& \ \ \ \ - \max_{\bm{x} \perp {D}_\gamma^{1/2} \bm{1}} \frac{\bm{x}^T \tilde{D}_\gamma^{-1/2}(\gamma I + A) \tilde{D}_\gamma^{-1/2} \bm{x}}{\bm{x}^T \bm{x}}  \\
& = \max_{\bm{y}\perp D \bm{1}} \frac{\bm{y}^T (\gamma D + A) \bm{y}}{(1+\gamma) \bm{y}^T D \bm{y}} + \min_{\bm{y} \perp D \bm{1}} \left(- \frac{\bm{y}^T (\gamma I + A) \bm{y}}{\bm{y}^T \tilde{D}_\gamma \bm{y}} \right)\\
& \geq  \min_{\bm{y} \perp D \bm{1}} \left(\frac{\bm{y}^T (\gamma D + A) \bm{y}}{(1+\gamma) \bm{y}^T D \bm{y}} - \frac{\bm{y}^T (\gamma I + A) \bm{y}}{\bm{y}^T \tilde{D}_\gamma \bm{y}} \right) \\
& = \min_{\bm{y} \perp D \bm{1}} \bigg(\frac{(\bm{y}^T (\gamma D + A) \bm{y}) (\bm{y}^T (\gamma I + D) \bm{y} )}{((1+\gamma) \bm{y}^T D \bm{y}) (\bm{y}^T (\gamma I + D) \bm{y})} \\
& \qquad\qquad - \frac{(\bm{y}^T (\gamma I + A) \bm{y}) (((1+\gamma) \bm{y}^T D \bm{y})) }{((1+\gamma) \bm{y}^T D \bm{y}) (\bm{y}^T (\gamma I + D) \bm{y})} \bigg)\\
& = \min_{\bm{y} \perp D \bm{1}} \frac{ \gamma \left( 1 + \frac{\bm{y}^T A \bm{y}}{\bm{y}^T D \bm{y}} \frac{\bm{y}^T \bm{y}}{\bm{y}^T D \bm{y}} - \frac{\bm{y}^T \bm{y}}{\bm{y}^T D \bm{y}} - \frac{\bm{y}^T A \bm{y}}{\bm{y}^T D \bm{y}}   \right) }{((1+\gamma) \bm{y}^T D \bm{y}) (\bm{y}^T (\gamma I + D) \bm{y}) / (\bm{y}^T D \bm{y})^2} \\
& = \min_{\bm{y} \perp D \bm{1}} 
\frac{\gamma \left( 1 - \frac{\bm{y}^T A \bm{y}}{\bm{y}^T D \bm{y}}\right ) \left(1- \frac{\bm{y}^T \bm{y}}{\bm{y}^T D \bm{y}}\right)}
{((1+\gamma) \bm{y}^T D \bm{y}) (\bm{y}^T (\gamma I + D) \bm{y}) / (\bm{y}^T D \bm{y})^2}
\geq 0\\
\end{align*}
\end{proof}

\section{Hyperparameters}
In this subsection, we report the optimal hyperparameters that are searched for FB-GNNs.
\begin{table}[htbp]
  \centering
  \small
  \caption{Hyperparameters for baseline models}
    \begin{tabular}{c|c|cccc|c}
    \toprule
    \textbf{Datasets} & \textbf{Models\textbackslash{}Hyperparameters} & \textbf{lr} & \textbf{weight\_decay} & \textbf{dropout} & \textbf{hidden} & \textbf{results} \\
    \midrule
    \multirow{4}[2]{*}{\textbf{Cornell}} &  GCN  & 0.05  & 5.00E-04 & 0.4   & 32    & 60.81 \\
          & GAT   & 0.005 & 5.00E-04 & 0.6   & 8     & 59.19 \\
          &  GraphSAGE & 0.1   & 1.00E-04 & 0.1   & 32    & 82.97 \\
          & MLP   & 0.05  & 1.00E-04 & 0.5   & 32    & 85.14 \\
    \midrule
    \multirow{4}[1]{*}{\textbf{Wisconsin}} &  GCN  & 0.05  & 5.00E-04 & 0.3   & 32    & 63.73 \\
          & GAT   & 0.005 & 5.00E-04 & 0.2   & 8     & 60.78 \\
          &  GraphSAGE & 0.1   & 1.00E-04 & 0.1   & 32    & 87.84 \\
          & MLP   & 0.05  & 1.00E-04 & 0.4   & 32    & 87.25 \\
          \midrule
    \multirow{4}[1]{*}{\textbf{Texas}} &  GCN  & 0.05  & 5.00E-05 & 0.4   & 32    & 61.62 \\
          & GAT   & 0.005 & 5.00E-04 & 0.1   & 8     & 59.73 \\
          &  GraphSAGE & 0.1   & 5.00E-04 & 0.2   & 32    & 82.43 \\
          & MLP   & 0.05  & 5.00E-04 & 0.3   & 32    & 84.59 \\
    \midrule
    \multirow{4}[2]{*}{\textbf{Film}} &  GCN  & 0.05  & 5.00E-04 & 0.3   & 32    & 30.98 \\
          & GAT   & 0.005 & 1.00E-04 & 0.2   & 8     & 29.71 \\
          &  GraphSAGE & 0.1   & 5.00E-04 & 0.3   & 32    & 35.28 \\
          & MLP   & 0.05  & 5.00E-05 & 0.9   & 32    & 36.08 \\
    \midrule
    \multirow{4}[2]{*}{\textbf{Chameleon}} &  GCN  & 0.05  & 5.00E-05 & 0.3   & 32    & 61.34 \\
          & GAT   & 0.005 & 1.00E-04 & 0.3   & 8     & 61.95 \\
          &  GraphSAGE & 0.1   & 5.00E-04 & 0.5   & 32    & 47.32 \\
          & MLP   & 0.05  & 5.00E-05 & 0.3   & 32    & 46.21 \\
    \midrule
    \multirow{4}[2]{*}{\textbf{Squirrel}} &  GCN  & 0.05  & 5.00E-05 & 0.6   & 32    & 41.86 \\
          & GAT   & 0.005 & 1.00E-04 & 0.2   & 8     & 43.88 \\
          &  GraphSAGE & 0.1   & 5.00E-05 & 0.6   & \multicolumn{1}{c}{32} & 30.16 \\
          & MLP   & 0.05  & 5.00E-05 & 0.4   & 32    & 29.39 \\
    \midrule
    \multirow{4}[2]{*}{\textbf{Cora}} &  GCN  & 0.05  & 5.00E-05 & 0.9   & 32    & 87.32 \\
          & GAT   & 0.005 & 1.00E-04 & 0.7   & 8     & 88.07 \\
          &  GraphSAGE & 0.1   & 5.00E-05 & 0.6   & 32    & 85.98 \\
          & MLP   & 0.05  & 5.00E-04 & 0.4   & 32    & 74.81 \\
    \midrule
    \multirow{4}[2]{*}{\textbf{Citeseer}} &  GCN  & 0.05  & 5.00E-04 & 0.5   & 32    & 76.7 \\
          & GAT   & 0.005 & 5.00E-04 & 0.6   & 8     & 76.42 \\
          &  GraphSAGE & 0.1   & 1.00E-04 & 0.6   & 32    & 77.07 \\
          & MLP   & 0.05  & 5.00E-05 & 0.6   & 32    & 73.45 \\
    \midrule
    \multirow{4}[2]{*}{\textbf{Pubmed}} &  GCN  & 0.05  & 5.00E-05 & 0.2   & 32    & 88.24 \\
          & GAT   & 0.005 & 5.00E-05 & 0.1   & 8     & 87.81 \\
          &  GraphSAGE & 0.1   & 5.00E-05 & 0.2   & 32    & 88.59 \\
          & MLP   & 0.05  & 1.00E-04 & 0.1   & 32    & 87.86 \\
    \midrule
    \multirow{4}[2]{*}{\textbf{﻿CitationFull\_dblp}} &  GCN  & 0.05  & 5.00E-05 & 0.8   & 32    & 85.87 \\
          & GAT   & 0.005 & 1.00E-04 & 0.3   & 8     & 85.89 \\
          &  GraphSAGE & 0.05  & 5.00E-05 & 0.2   & 32    & 81.19 \\
          & MLP   & 0.05  & 5.00E-04 & 0.3   & 32    & 77.39 \\
    \midrule
    \multirow{4}[2]{*}{\textbf{ Coauthor\_CS}} &  GCN  & 0.05  & 5.00E-05 & 0.2   & 32    & 93.91 \\
          & GAT   & 0.005 & 5.00E-05 & 0.2   & 8     & 93.41 \\
          &  GraphSAGE & 0.1   & 5.00E-05 & 0.2   & 32    & 94.38 \\
          & MLP   & 0.05  & 5.00E-05 & 0.2   & 32    & 93.72 \\
    \midrule
    \multirow{4}[2]{*}{\textbf{ Coauthor\_Physics}} &  GCN  & 0.05  & 5.00E-05 & 0.1(0.3) & 32    & 96.84 \\
          & GAT   & 0.005 & 5.00E-04 & 0.5   & 8     & 96.32 \\
          &  GraphSAGE & -     & -     & -     & -     & OOM \\
          & MLP   & 0.05  & 5.00E-05 & 0.5(0.6) & 32    & 95.77 \\
    \midrule
    \multirow{4}[2]{*}{\textbf{ Amazon\_Computers}} &  GCN  & 0.05  & 5.00E-05 & 0.1   & 32    & 87.03 \\
          & GAT   & 0.005 & 5.00E-05 & 0.2   & 8     & 89.74 \\
          &  GraphSAGE & 0.1   & 5.00E-05 & 0.1   & 32    & 83.7 \\
          & MLP   & 0.05  & 5.00E-05 & 0.2   & 32    & 83.89 \\
    \midrule
    \multirow{4}[2]{*}{\textbf{ Amazon\_Photo }} &  GCN  & 0.05  & 5.00E-05 & 0.2   & 32    & 93.61 \\
          & GAT   & 0.005 & 5.00E-05 & 0.2   & 8     & 94.12 \\
          &  GraphSAGE &       &       &       &       &  \\
          & MLP   & 0.05  & 5.00E-05 & 0.4   & 32    & 90.87 \\
    \bottomrule
    \bottomrule
    \end{tabular}%
  \label{tab:hyperparameters_baseline}%
\end{table}%

\begin{table}[htbp]
\tiny
  \caption{Hyperparameters for FB-GNNs}
    \begin{tabular}{c|c|ccccc|ccc}
    \toprule
    \toprule
    \multirow{2}[4]{*}{\textbf{Datasets}} & \multirow{2}[4]{*}{\textbf{Models\textbackslash{}Hyperparameters}} & \multicolumn{5}{c|}{\textbf{symmetric renormalized adjacency matrix}} & \multicolumn{3}{c}{\textbf{lazy random walk matrix}} \\
\cmidrule{3-10}          &       & lr    & weight\_decay & dropout & hidden & results & weight\_decay & dropout & results \\
    \midrule
    \multirow{4}[2]{*}{\textbf{Cornell}} & MF-GCN & 0.05  & 1.00E-03 & 0.3   & 32    & 83.78 & 5.00E-04 & 0.3   & 85.14 \\
          & MF-GAT & 0.05  & 5.00E-04 & 0.2   & 8     & 87.3  & 5.00E-04 & 0.3   & \textbf{88.92} \\
          & MF-GraphSAGE & 0.05  & 5.00E-04 & 0.1   & 32    & 86.44 & 5.00E-04 & 0.1   & 86.49 \\
          & MF-Geom-GCN* & 0.05  & 5.00E-04 & 0.3   & 32    & 82.99 & 5.00E-04 & 0.3   & 83.41 \\
    \midrule
    \multirow{4}[1]{*}{\textbf{Wisconsin}} & MF-GCN & 0.05  & 5.00E-04 & 0.1   & 32    & 87.45 & 5.00E-04 & 0.4   & 87.25 \\
          & MF-GAT & 0.05  & 5.00E-04 & 0.3   & 8     & 85.88 & 5.00E-04 & 0.2   & \textbf{89.22} \\
          & MF-GraphSAGE & 0.05  & 5.00E-04 & 0.2   & 32    & 88.43 & 5.00E-04 & 0.3   & 87.45 \\
          & MF-Geom-GCN* & 0.05  & 5.00E-04 & 0.3   & 32    & 85.66 & 5.00E-04 & 0.3   & 86.1 \\
          \midrule
    \multirow{4}[1]{*}{\textbf{Texas}} & MF-GCN & 0.05  & 5.00E-04 & 0.1   & 32    & 84.59 & 1.00E-03 & 0.3   & 86.49 \\
          & MF-GAT & 0.05  & 1.00E-04 & 0.6   & 8     & 82.7  & 5.00E-04 & 0.4   & \textbf{88.65} \\
          & MF-GraphSAGE & 0.1   & 5.00E-04 & 0.2   & 32    & 86.22 & 5.00E-04 & 0.1   & 87.3 \\
          & MF-Geom-GCN* & 0.05  & 5.00E-04 & 0.3   & 32    & 83.41 & 5.00E-04 & 0.3   & 84.41 \\
    \midrule
    \multirow{4}[1]{*}{\textbf{Film}} & MF-GCN & 0.05  & 5.00E-03 & 0.2   & 32    & 35.47 & 5.00E-03 & 0.2   & 36.11 \\
          & MF-GAT & 0.05  & 5.00E-04 & 0.5   & 8     & 36    & 5.00E-04 & 0.5   & \textbf{36.83} \\
          & MF-GraphSAGE & 0.05  & 5.00E-04 & 0.1   & 32    & 35.88 & 5.00E-05 & 0.4   & 35 \\
          & MF-Geom-GCN* & 0.05  & 5.00E-05 & 0.6   & 32    & 34.26 & 5.00E-05 & 0.7   & 34.08 \\
          \midrule
    \multirow{4}[1]{*}{\textbf{Chameleon}} & MF-GCN & 0.05  & 5.00E-05 & 0.7   & 32    & \textbf{65.66} & 5.00E-05 & 0.7   & 61.07 \\
          & MF-GAT & 0.005 & 5.00E-04 & 0.5   & 8     & 63.75 & 5.00E-04 & 0.4   & 62.8 \\
          & MF-GraphSAGE & 0.05  & 5.00E-04 & 0.6   & 32    & 48.11 & 5.00E-04 & 0.6   & 48.57 \\
          & MF-Geom-GCN* & 0.05  & 5.00E-05 & 0.8   & 32    & 63.8  & 5.00E-05 & 0.8   & 62.23 \\
    \midrule
    \multirow{4}[1]{*}{\textbf{Squirrel}} & MF-GCN & 0.05  & 5.00E-05 & 0.6   & 32    & \textbf{50.56} & 5.00E-05 & 0.6   & 46.28 \\
          & MF-GAT & 0.005 & 5.00E-05 & 0.5   & 8     & 44.07 & 5.00E-04 & 0.5   & 43.56 \\
          & MF-GraphSAGE & 0.05  & 5.00E-04 & 0.5   & 32    & 33.24 & 5.00E-04 & 0.6   & 30.44 \\
          & MF-Geom-GCN* & 0.05  & 5.00E-05 & 0.7   & 32    & 40.02 & 5.00E-05 & 0.8   & 39.02 \\
          \midrule
    \multirow{4}[1]{*}{\textbf{Cora}} & MF-GCN & 0.05  & 5.00E-04 & 0.8   & 32    & 87.34 & 5.00E-04 & 0.7   & 87.46 \\
          & MF-GAT & 0.05  & 5.00E-05 & 0.6   & 8     & 87.34 & 1.00E-04 & 0.6   & 87.28 \\
          & MF-GraphSAGE & 0.05  & 5.00E-05 & 0.7   & 32    & 86.62 & 1.00E-04 & 0.6   & 86.35 \\
          & MF-Geom-GCN* & 0.05  & 1.00E-04 & 0.6   & 32    & \textbf{87.81} & 1.00E-04 & 0.7   & 87.3 \\
    \midrule
    \multirow{4}[2]{*}{\textbf{Citeseer}} & MF-GCN & 0.05  & 5.00E-03 & 0.3   & 32    & 76.42 & 5.00E-03 & 0.3   & 76.24 \\
          & MF-GAT & 0.05  & 1.00E-04 & 0.6   & 8     & 76.45 & 5.00E-04 & 0.6   & 76.92 \\
          & MF-GraphSAGE & 0.05  & 5.00E-05 & 0.7   & 32    & 78.01 & 5.00E-05 & 0.7   & 76.48 \\
          & MF-Geom-GCN* & 0.05  & 5.00E-04 & 0.6   & 32    & \textbf{78.02} & 5.00E-04 & 0.7   & 77.02 \\
    \midrule
    \multicolumn{1}{c|}{\multirow{3}[2]{*}{\textbf{Pubmed}}} & MF-GCN & 0.05  & 5.00E-04 & 0.3   & 32    & 89.74 & 5.00E-04 & 0.2   & 89.57 \\
          & MF-GAT & 0.05  & 5.00E-05 & 0.3   & 8     & 89.03 & 5.00E-05 & 0.4   & 89.43 \\
          & MF-GraphSAGE & 0.05  & 5.00E-05 & 0.3   & 32    & 89.05 & 5.00E-05 & 0.3   & 88.58 \\
    \midrule
    \multirow{3}[2]{*}{\textbf{﻿CitationFull\_dblp}} & MF-GCN & 0.05  & 5.00E-05 & 0.6   & 32    & \textbf{85.9} & 0.00E+00 & 0.6   & 85.51 \\
          & MF-GAT & 0.05  & 5.00E-05 & 0.6   & 8     & 84.94 & 5.00E-05 & 0.5   & 85.35 \\
          & MF-GraphSAGE & 0.05  & 5.00E-05 & 0.3   & 32    & 85.66 & 5.00E-05 & 0.6   & 85.02 \\
    \midrule
    \multirow{3}[2]{*}{\textbf{ Coauthor\_CS}} & MF-GCN & 0.05  & 1.00E-04 & 0.3   & 32    & \textbf{95.33} & 1.00E-04 & 0.4   & 95.31 \\
          & MF-GAT & 0.05  & 5.00E-05 & 0.4   & 8     & 94.13 & 5.00E-05 & 0.5   & 95.3 \\
          & MF-GraphSAGE & 0.05  & 5.00E-05 & 0.3   & 32    & 94.97 & 5.00E-05 & 0.5   & 95 \\
    \midrule
    \multirow{3}[2]{*}{\textbf{ Coauthor\_Physics}} & MF-GCN & 0.05  & 5.00E-05 & 0.4   & 32    & 97.03 & 5.00E-05 & 0.4   & \textbf{97.07} \\
          & MF-GAT & -     & -     & -     & -     & OOM   & -     & -     & OOM \\
          & MF-GraphSAGE & -     & -     & -     & -     & OOM   & -     & -     & OOM \\
    \midrule
    \multirow{3}[2]{*}{\textbf{ Amazon\_Computers}} & MF-GCN & 0.05  & 1.00E-05 & 0.4   & 32    & \textbf{91.54} & 5.00E-05 & 0.4   & 91.32 \\
          & MF-GAT & 0.05  & 5.00E-05 & 0..2  & 8     & 89.57 & 5.00E-05 & 0.3   & 91.1 \\
          & MF-GraphSAGE & 0.05  & 5.00E-05 & 0.6   & 32    & 88.01 & 5.00E-05 & 0.5   & 87.12 \\
    \midrule
    \multirow{3}[2]{*}{\textbf{ Amazon\_Photo }} & MF-GCN & 0.05  & 5.00E-05 & 0.4   & 32    & \textbf{95.57} & 1.00E-04 & 0.3   & 95.53 \\
          & MF-GAT & 0.05  & 1.00E-04 & 0.2   & 8     & 94.58 & 1.00E-04 & 0.4   & 94.88 \\
          & MF-GraphSAGE & 0.05  & 5.00E-05 & 0.5   & 32    & 91.4  & 5.00E-05 & 0.6   & 91.02 \\
    \bottomrule
    \bottomrule
    \end{tabular}%
  \label{tab:hyperparameters_fbgnn}
\end{table}%

\end{document}